\documentclass[twoside, onecolumn]{article}
\usepackage[margin=1in]{geometry}

\usepackage{natbib}

\usepackage{amsmath,amssymb,latexsym}
\usepackage{amsthm}
\usepackage{graphicx}
\usepackage{caption}
\usepackage{subcaption}
\usepackage[ruled]{algorithm2e}
\usepackage[usenames, dvipsnames]{color}
\usepackage{tikz}
\usetikzlibrary{shapes,arrows}
\tikzstyle{block} = [ rectangle, draw, fill=white, text width=5em, text centered, rounded corners, minimum height=4em ]
\tikzstyle{line} = [ draw, -latex' ]

\usepackage{xr}
\makeatletter
\newcommand*{\addFileDependency}[1]{
  \typeout{(#1)}
  \@addtofilelist{#1}
  \IfFileExists{#1}{}{\typeout{No file #1.}}
}
\makeatother

\def\R{\mathbb{R}}

\def\bc{\mathbf{c}}

\DeclareMathOperator*{\argmin}{argmin}

\newtheorem{lemma}{Lemma}
\newtheorem{theorem}{Theorem}

\newtheorem{definition}{Definition}
\newtheorem{assumption}{Assumption}
\newtheorem{remark}{Remark}

\begin{document}

%

%

\onecolumn

\title{Compressing Large Sample Data for Discriminant Analysis}
\author{Alexander F. Lapanowski
\and
Irina Gaynanova
}
\date{Texas A\&M University\\
\texttt{\{alapanow, irinag\}@stat.tamu.edu}}
\maketitle

\begin{abstract}
Large-sample data became prevalent as data acquisition became cheaper and easier. While a large sample size has theoretical advantages for many statistical methods, it presents computational challenges. Sketching, or compression, is a well-studied approach to address these issues in regression settings, but considerably less is known about its performance in classification settings. Here we consider the computational issues due to large sample size within the discriminant analysis framework. We propose a new compression approach for reducing the number of training samples for linear and quadratic discriminant analysis, in contrast to existing compression methods which focus on reducing the number of features. We support our approach with a theoretical bound on the misclassification error rate compared to the Bayes classifier. Empirical studies confirm the significant computational gains of the proposed method and its superior predictive ability compared to random sub-sampling.

\end{abstract}

\section{Introduction}

Linear Discriminant Analysis (LDA) \citep{mardia79} is a linear classification rule which separates the classes by maximizing between-class variability compared to within-class variability. Applying LDA requires constructing the within-class covariance matrix, which has complexity $O(n\,p^{2})$ in the number of training samples $n$ and number of features $p$. As large-sample data acquisition became prevalent, it became computationally expensive to apply LDA to such data even for moderately-sized $p$.

Compression \citep{boutsidis2009random, pilanci2015randomized,pilanci2016iterative, vempala2005random, mahoney2011randomized}, or sketching, is a popular approach for scaling algorithms to large data. Given the training data $X\in \mathbb{R}^{n\times p}$, compression uses a random matrix $Q$ to either reduce the number of rows (samples) or columns (features) in $X$. The corresponding reduced-size $QX$ or $XQ$ is called a sketch of the original $X$. The sketch is used in place of $X$ to approximate the solution of the full algorithm. For example, compression is used in least-squares regression 
\citep{drineas2011faster, mahoney2011randomized}; 
non-negative least-squares regression \citep{boutsidis2009random}; ridge regression \citep{wang2017sketched, homrighausen2019compressed} and $\ell^1$-penalized regression \citep{zhou2008compressed}. Compression for a broader class of convex minimization problems  is considered in \citet{pilanci2016iterative}.

Despite the widespread use of compression in regression contexts, and considerable progress in theoretical understanding of its performance in regression, compression for sample reduction has not been widely used in discriminant analysis.  Additionally, existing results on compression due to large $n$ in the regression literature \citep{wang2017sketched, homrighausen2019compressed} can not be applied to discriminant analysis. In regression, the training data $X\in \mathbb{R}^{n \times p}$ is treated as fixed, with continuous response $Y\in \mathbb{R}^{n}$ modeled conditionally on $X$. In contrast, in discriminant analysis the observations in $X\in \mathbb{R}^{n \times p}$ are treated as random, and are modeled conditionally on the discrete class membership $Y\in \{1,2\}^n$. Thus, the theoretical analysis of sample compression in LDA requires different techniques than for regression.

There is an extensive body of work considering feature compression in LDA, such as \cite{li2019one,tu2014making, chowdhury2018randomized, durrant2010compressed, durrant2012tight}.  However, these works focus on reducing the number of features $p$ while keeping the number of samples $n$ fixed. They do not consider the case where the computational bottleneck is due to the large number of samples $n$.

In this work, we address these challenges and bridge the existing gap between compression with large $n$ in regression and compression with large $n$ in discriminant analysis. Our work makes the following contributions:
\begin{itemize}
    \item We develop a new method, Compressed LDA, for large sample data that is based on \textit{separate} compression within each class in contrast to joint compression of existing approaches \citep{ye2017fast};
    \item We derive a finite-sample bound on misclassification error rate of Compressed LDA compared to the optimal error rate of the Bayes classifier;
    \item We extend Compressed LDA to the setting with unequal class covariance matrices leading to Compressed Quadratic Discriminant Analysis (QDA) \citep{hastie_elements_2009}, to our knowledge this it the first method that considers compression within the QDA context;
    \item We demonstrate significant computational advantages of our methods compared to discriminant analysis on the full data and their superior classification performance compared to methods based on random sub-sampling or joint compression \cite{ye2017fast}.
\end{itemize}

\subsection{Related Works}\label{sec:RelatedWorks}

Existing works on compression in LDA \citep{li2019one,tu2014making} focus on reducing the number of features $p$, and thus do not consider the case where the computational bottleneck is due to the large number of samples $n$. To our knowledge, the only exception is the Fast Random Fisher Discriminant Analysis (FRF) of \citet{ye2017fast}.

In \citet{ye2017fast}, the authors use joint compression of classes to form a sketch $QX\in \mathbb{R}^{m \times p}$, $m \ll n$, via a random matrix $Q\in \mathbb{R}^{m \times n}$, and then use the sketch within the generalized eigenvalue formulation of LDA to form the approximate discriminant vector $\beta_{\bc}\in \mathbb{R}^p$. The discriminant vector is applied to form the projected training data $\beta^{\top}_{\bc}\mathbf{x}_i\in \mathbb{R}$, which is used to train LDA instead of original $\mathbf{x}_i\in \mathbb{R}^p$. The $m$ compressed samples in $QX\in \mathbb{R}^{m \times p}$ are thus only used to form $\beta_c$. This is because these $m$ samples can not be assigned class labels, as multiplication by $Q$ allows mixing of both classes. Furthermore, due to this mixing, it is not possible to form class-specific covariance matrices based on compressed samples in $QX$, and thus the method of \citet{ye2017fast} cannot be extended to QDA.  In contrast, our method applies separate class compression, not only allowing an extension to QDA, but also leading to significantly better empirical performance (in terms of both lower error rate and lower variance).

Another difference between our work and the work of \citet{ye2017fast} is the corresponding theoretical analysis.  In \citet{ye2017fast}, the authors compare the compressed discriminant vector $\beta_{\mathbf{c}}$ to the discriminant vector $\widehat{\beta}$ based on the full data by deriving the bound on the difference of projection values
$
|(\mathbf{x}-\overline{X})^\top (\beta_{\mathbf{c}}-\widehat{\beta}) |,
$
where $\overline{X}= n^{-1}\sum_{i=1}^{n} \mathbf{x}_{i}$ is the training sample mean and  $\mathbf{x}\in \mathbb{R}^{p}$ is a random test sample. It is unclear, however, whether this bound directly translates into a similar difference in misclassficiation error rates, which is a more natural loss within a classification context. Furthermore, since the bound is provided with respect to $\widehat {\beta}$ rather than the true population $\beta^*$, it is unclear how the performance of the method of \citet{ye2017fast} compares to the performance of the Bayes classifier. In contrast, we directly analyze the misclassification error rate of the proposed Compressed LDA method, and we derive a finite-sample bound on its rate compared to the Bayes classifier.

In the regression literature on compression, the quality of the compressed solution $\beta_c$ is typically evaluated either by bounding mean-squared error compared to the underlying true parameter vector $\beta^*$ \citep{homrighausen2019compressed}, or by considering the $\varepsilon$-optimality. Let $f$ be the objective function that is minimized within the given algorithm (e.g. standard least-squares, $\ell_1$-penalized least-squares, etc.) over some subset $S$ of $\mathbb{R}^{p}$, where the function $f$ is based on the full training data. 
The compressed solution $\beta_{\mathbf{c}}$ is said to be $\varepsilon$-optimal \citep{sarlos2006improved, mahoney2011randomized} if
$$
\min_{\beta\in S}f(\beta)\leq f(\beta_{\mathbf{c}})\leq (1+\varepsilon)^{2}\,\min_{\beta\in S}f(\beta).
$$
While $\varepsilon$-optimality is natural in a regression context, where the loss in the objective function represents the sample average of targeted population loss, LDA solves a generalized eigenvalue problem rather than directly minimizing the misclassification error rate. Thus, bounding the misclassification error rate of Compressed LDA directly in terms of the Bayes error rate provides a more direct answer regarding its theoretical performance, and it is consistent with results in the LDA literature without compression \citep{mclachlan2004discriminant,shao2011sparse, bickel2004some}.



Another sample size reduction method outside of compression is squashing \citep{dumouchel1999squashing,madigan2002likelihood, pavlov2000towards}, 
which partitions the $n$ training samples into $d$ distinct segments, calculates a fixed number of moments $k$ for each segment, and then generates a smaller number of new samples within each segment preserving the corresponding original moments. Each new sample comes with a weight that accounts for a possible discrepancy between the distribution of samples across segments in the original data and the distribution of samples across segments in the new data. Because of the weights, one can not simply apply LDA to the new "squashed" data, as the weights will need to be included to modify the estimation algorithm. Furthermore, while squashing reduces the number of training samples, its computational complexity depends on the number of partitions $d$, number of calculated moments $k$, and the number of newly-generated samples. Since partitioning the data may lead to an exponential number  of segments $d$ in the number of features $p$, applying squashing in LDA context may be more computationally expensive than training LDA on the full data, and thus we do not pursue this approach here. 

\subsection{Notation}
For a vector $v\in \mathbb{R}^{p}$, we let $\|v\|_{2}$ be the Euclidean norm $\sqrt{\sum_{i=1}^{p}|v_{i}|^{2}}$. For a matrix $M\in \mathbb{R}^{k \times p}$, we let $M_{i,j}$ be its $(i,j)$-th element,  $\|M\|_{\text{op}}= \sup_{\|v\|_{2}\leq 1} \|Mv\|_{2}$ be its operator norm, and  $\|M\|_{F}= \sqrt{\sum_{i,j} |M_{i,j}|^{2}}$ be the Frobenius norm. For a random variable $Z$, we let $\|Z\|_{\Psi_2}= \inf\{t>0: \mathbb{E}\exp(Z^{2}/t^{2}) \leq 2\}$ be its sub-Gaussian norm and $\|Z\|_{\Psi_1} = \inf\{t>0: \mathbb{E}\exp(|Z|/t)\leq 2\}$ its sub-Exponential norm. We use $\Phi(\cdot)$ and $\phi(\cdot)$ to denote the cdf and the pdf of the standard normal distribution, respectively.

\section{Compressed Linear Discriminant Analysis}

\subsection{Review of LDA}

Let $\{(\mathbf{x}_i, y_i)\}_{i=1}^{n}$ be independent pairs of feature vectors $\mathbf{x}_i\in \mathbb{R}^{p}$ and labels $y_i \in \{1,2\}$. Let $X=
\begin{pmatrix}
{X^{1}}^{\top}&
{X^{2}}^{\top}
\end{pmatrix}^{\top}$ be the corresponding $n\times p$ matrix of training samples, where $X^{g}\in \mathbb{R}^{n_{g} \times p}$ is the sub-matrix consisting of $n_g$ samples $\mathbf{x}_i^{g}$ belonging to class $g=1,2$.
Let $Y= (\{1\}^{n_1}, \{2\}^{n_2})^{\top}$ be the corresponding vector of class labels.  We let $\overline{X}:= n^{-1}\sum_{i=1}^{n} \mathbf{x}_{i}$ be the overall training sample mean, and let $\overline{X}_{g}$ be the $g$th class sample mean $n_{g}^{-1}\sum_{i=1}^{n_g} \mathbf{x}_{i}^{g}.$ We use the following standard assumption \citep{mardia79}. 
\begin{assumption}\label{assump:Normality}
Conditional on class membership $g=1,2$, the samples $\mathbf{x}_{i}^{g} \in \mathbb{R}^{p}$ are i.i.d. $N(\mu_g\,,\, \Sigma_{w})$. 
\end{assumption}

Linear Discriminant Analysis \citep[Chapter 11]{mardia79} seeks a vector $\beta\in \mathbb{R}^{p}$ such that the values $ \beta^\top \mathbf{x}_{i}^{g}$ are well-separated between classes. Given the within-class covariance matrix
\begin{equation}\label{eq:sigmaw}
\widehat{\Sigma}_{w}:= \frac{1}{n}\sum_{g=1}^{2}\sum_{i=1}^{n_g}(\mathbf{x}_{i}^{g}-\overline{X}_{g})(\mathbf{x}_{i}^{g}-\overline{X}_{g})^{\top},
\end{equation}
and vector of the class mean differences 
\begin{equation}\label{eq:d}
d:= \frac{\sqrt{n_1 n_2}}{n}\{\overline{X}_{1}-\overline{X}_{2}\},
\end{equation}
LDA estimates $\beta$ as
$
\widehat{\beta}:= \widehat{\Sigma}_{w}^{-1}d
$ \citep[Theorem 11.5.1]{mardia79}.

Given the estimated discriminant vector $\widehat \beta \in \mathbb{R}^{p}$, the LDA classification rule labels a new $\mathbf{x}\in \mathbb{R}^{p}$ by minimizing

\begin{align}
\begin{split}\label{eq:ClassificationRule}
\argmin_{g=1,2}\,\big\{(\mathbf{x}-\overline{X}_{g})^\top \widehat \beta\,(\widehat\beta^{\top}\widehat{\Sigma}_{w}
\widehat \beta)^{-1} \widehat\beta^\top (\mathbf{x}-\overline{X}_{g})-2\log(n_{g}/n)\big\}.
\end{split}
\end{align}

\begin{remark}
Decision rule \eqref{eq:ClassificationRule} is the Fisher's Discriminant Analysis rule \cite[Section 11.5]{mardia79}. For the two-group case and normally distributed data, it is equivalent to the likelihood decision rule \cite[Section 4.3]{hastie_elements_2009} as discussed in \citet[Section 11.5]{mardia79}. 
\end{remark}

Constructing $\widehat{\Sigma}_{w}$ in~\eqref{eq:sigmaw} has complexity $O(n\,p^2)$, which makes LDA computationally expensive on large-sample data.


\subsection{Compressed LDA}\label{sec:compLDA}
 Our goal is to reduce the computational complexity of LDA while maintaining its classification performance. To achieve this, we propose to separately compress each class of training data $X^g\in \mathbb{R}^{n_g\times p}$ via a sparse rademacher matrix $Q^{g}\in \mathbb{R}^{m_g\times n_g}$ as defined below.

\begin{definition}\label{def:CompressionMatrix}
A matrix $Q^{g}\in \mathbb{R}^{m_g\times n_g}$ is a \emph{sparse rademacher} matrix with parameter $s\in (0,1)$ if the elements $Q^{g}_{j, k}$  are i.i.d. with distribution
$$
\mathbb{P}(Q^{g}_{j,k}=1) = \mathbb{P}(Q^{g}_{j,k}=-1) = \frac{s}{2}, \,\mathbb{P}(Q^{g}_{j,k}=0)=1-s. 
$$
\end{definition}

\begin{definition}\label{def:CompressedData}
The $j$-th \emph{compressed data sample} in class $g$ is 
\begin{equation}\label{eq:CompressedSample}
    \mathbf{x}_{j, \mathbf{c}}^{g} = \frac{1}{\sqrt{n_g\,s}}\sum_{i=1}^{n_g}Q^{g}_{j,i}( \mathbf{x}_{i}^{g}-\overline{X}_{g})+\overline{X}_{g},
\end{equation}
where $Q_{j,i}^{g}$ are entries of the sparse rademacher matrix  $Q^{g}\in \mathbb{R}^{m_g \times n_g}$ of Definition~\ref{def:CompressionMatrix}.
\end{definition}

The compressed samples \eqref{eq:CompressedSample} are efficiently computed due to the sparse matrix structure for $Q^g$: only the non-zero entries of each row and corresponding samples $\mathbf{x}_{i}^{g}$ appear in the summation.

\begin{definition}\label{def:MeanAndCov}
The compressed within-class sample covariance matrix $\widehat{\Sigma}_{w, \mathbf{c}}\in \mathbb{R}^{p\times p}$ is defined as the within-class sample covariance matrix of the compressed $\mathbf{x}_{j,\mathbf{c}}^g$
\begin{equation}\label{eq:Sigmawc}
\widehat{\Sigma}_{w, \mathbf{c}}:= \frac{1}{m}\sum_{g=1}^{2}\sum_{j=1}^{m_g}(\mathbf{x}_{j, \mathbf{c}}^{g}-\overline{X}_{g})(\mathbf{x}_{j, \mathbf{c}}^{g}-\overline{X}_{g})^{\top}.
\end{equation}
The \emph{compressed discriminant vector} is $\beta_{\mathbf{c}}:=\widehat{\Sigma}_{w, \mathbf{c}}^{-1}d$, where $d$ is defined as in~\eqref{eq:d}.
\end{definition}

\begin{algorithm}[!t]
 \SetAlgoLined
 \caption{Compressed LDA}\label{a:algorithm}
  \SetKwInOut{Input}{Input}
  \SetKwInOut{Output}{Output}
\DontPrintSemicolon 
  
\Input{$X\in \R^{n\times p}$, $Y\in \R^{n}$,  $s\in (0,1)$, $m \ll n$}
\Output{$\beta_{\mathbf{c}}\in \mathbb{R}^{p}$, $\widehat{\Sigma}_{w, \mathbf{c}}\in \mathbb{R}^{p\times p}$}

Compute $\overline{X}_{g}$, $g=1,2$, and $d$ as in~\eqref{eq:d}\;
Set $m_g = \lfloor n_gm/n \rfloor$, $g=1,2$.\;
Form compressed samples $\mathbf{x}_{j,\bc}^g$\;
Form $\widehat{\Sigma}_{w, \mathbf{c}}\in \mathbb{R}^{p\times p}$ as in~\eqref{eq:Sigmawc}\;
Set $\beta_{\mathbf{c}} = \widehat{\Sigma}_{w, \mathbf{c}}^{-1} d$\;
Use $\beta_{\mathbf{c}}$, $\widehat{\Sigma}_{w,\mathbf{c}}$ in rule \eqref{eq:ClassificationRule} instead of $\widehat \beta$, $\widehat \Sigma_w$\;
\Return{$\beta_{\mathbf{c}}= \widehat{\Sigma}_{w, \mathbf{c}}^{-1}d$}
\end{algorithm}

The proposed Compressed LDA classifies a new $\mathbf{x}\in \mathbb{R}^p$ as in~\eqref{eq:ClassificationRule}, with $\widehat \beta$ and $\widehat{\Sigma}_{w}$ replaced by $\beta_{\mathbf{c}}$, and $\widehat{\Sigma}_{w, \mathbf{c}}$. Algorithm~\ref{a:algorithm} summarizes the full workflow for Compressed LDA.

Our proposed compression scheme is analogous to partial compression within the compressed regression literature, see e.g. Section 2.1 of \citet{homrighausen2019compressed}. 
Given the matrix of covariates $X\in \mathbb{R}^{n\times p}$ and response $Y\in \mathbb{R}^{n}$, partial compression calculates the inner-product $X^{\top}Y$ on the full data and only uses compression to approximate 
$X^{\top}X$. The rationale is that calculating $X^{\top}Y$ only has complexity $O(n\,p)$ compared to complexity $O(n\,p^2)$ for calculating $X^{\top}X$. Similarly in discriminant analysis, calculating $d$ on the full data only has complexity $O(n\,p)$, whereas calculating $\widehat \Sigma_w$ has complexity $O(n\,p^2)$, and thus we only use compression to approximate the latter term.

The proposed compression scheme has several advantages. First, by compressing the classes individually, we are able to unambiguously assign labels to the compressed samples, thus allowing us to form the compressed within-class covariance matrix. This is not possible with the method of \citet{ye2017fast}, which allows mixing samples from both classes in one compressed sample. Secondly, using sparse compression matrices leads to both memory and computational advantages compared to e.g. random Gaussian compression matrices. Due to sparsity, the average complexity of data compression~\eqref{eq:CompressedSample} is $O(nmps)$ rather than $O(nmp)$ for dense matrices. Thus, the overall average complexity of data compression and construction of $\widehat{\Sigma}_{w,\mathbf{c}}$ is $O(nmps+mp^2)$ compared to the complexity $O(np^2)$ of LDA on the full data. Choosing $m$ and $s$ so that $ms \ll p$ ensures that Compressed LDA is faster than full LDA. The computational costs of compression~\eqref{eq:CompressedSample} can be further reduced by parallelizing the construction of $Q^{g}X^{g}$.

\section{Error bound of Compressed LDA}
In this section we derive a bound on the misclassification error rate of Compressed LDA compared to the optimal rate of the Bayes classifier. To our knowledge, this is the first such result for a sample compression method within the discriminant analysis framework.

We next define the Bayes classifier, which gives the optimal (minimal) error rate under Assumption~\ref{assump:Normality}.
\begin{definition}
Under Assumption~\ref{assump:Normality}, and for equal prior class probabilities $\pi_1 = \pi_2$,
the \emph{Bayes decision rule} classifies $\mathbf{x}\in\mathbb{R}^{p}$ to class $1$ if and only if
$
\delta^{\top} \Sigma_{w}^{-1}(\mathbf{x}-\mu)\geq 0,
$
where $\delta = (\mu_1-\mu_2)/2$, and $\mu = (\mu_1 + \mu_2)/2$.
\end{definition}
The corresponding optimal misclassification error rate is given by \citep[Chapter 11.6]{mardia79}
\begin{equation}\label{eq:Ropt}
R_{\text{opt}}:=\Phi(-\sqrt{\delta^\top \Sigma_{w}^{-1}\delta}).
\end{equation}
We consider the case of equal prior class probabilities for clarity of technical derivations, which focus on the effects of compression. For the same reason, we assume equality of class sizes and their corresponding compression dimensions.
\begin{assumption}\label{assump:EqualClass}
$n_1 = n_2 = n/2$ and $m_1 = m_2 = m/2$.
\end{assumption}
These assumptions can be relaxed at the expense of more technical proofs without affecting the resulting rates, e.g. Hoeffding inequality bounds $n_g/n$ in terms of $\pi_g$ with rate $O(n^{-1/2})$. 
Appendix~\ref{appen:Extension} contains further details regarding this extension.

We next bound the misclassification error rate of the proposed Compressed LDA in Section~\ref{sec:compLDA} in terms of the optimal rate $R_{\text{opt}}$ in~\eqref{eq:Ropt}. Under Assumption~\ref{assump:EqualClass}, the Compressed LDA rule assigns new $\mathbf{x}$ to class 1 if and only if $d^\top\widehat{\Sigma}_{w,\mathbf{c}}^{-1}(\mathbf{x}-\overline{X})\geq 0.$ Under Assumptions~\ref{assump:Normality}-\ref{assump:EqualClass}, by \citet[Section~2]{shao2011sparse}, the corresponding error rate of Compressed LDA is given by 
\begin{equation}\label{eq:COSerror}
R_{\mathbf{c}}=\frac{1}{2}\sum_{g=1}^{2} 
\Phi \bigg( \frac{d^{\top}\widehat{\Sigma}_{w,\mathbf{c}}^{-1}\{(-1)^{g}(\mu_{g}-\overline{X}_{g})- d\}
}{\sqrt{d^{\top}\widehat{\Sigma}_{w,\mathbf{c}}^{-1}\, \Sigma_{w}\,\widehat{\Sigma}_{w,\mathbf{c}}^{-1} d}}\bigg).
\end{equation}
 
We now state our main result.

\begin{theorem}\label{thm:MisclassificationError}
Under Assumptions~\ref{assump:Normality} and \ref{assump:EqualClass}, and for  $\pi_1 = \pi_2$, there exists an absolute constant $C>0$ such that with probability at least $1-\eta,$
$$
    |R_{\mathbf{c}} - R_{\text{opt}}|  \leq
   C \,P\, K_{s}^{2}\,\sqrt{\frac{\log(\eta^{-1})+p}{m}},
$$
where $P=\phi(\sqrt{\delta^\top \Sigma_{w}^{-1} \delta})\,(\sqrt{\delta^\top \Sigma_{w}^{-1}\delta}+1)$, and
 $K_{s}^{2}=[s \log\big(1+s^{-1}\big)]^{-1}$. 
\end{theorem}

The upper bound depends on the sparsity level $s$ through $K_{s}$, which appears in the proofs as the sub-Gaussian norm of the elements of $Q^g/\sqrt{s}$ (see Lemma~5 in the Supplement). As $s\to 0$, fewer training samples are used when forming each compressed sample, and the upper bound of Theorem~\ref{thm:MisclassificationError} increases. As $s\to 1$, more training samples are included, and the upper bound decreases. However, as $s$ increases so does the run time for Compressed LDA. Thus, there is a trade-off between accuracy and speed determined by $s$. 

Existing results in the LDA literature (i.e. \citet{shao2011sparse}) have error rates $\mathcal{O}_{p}(n^{-1/2})$. Since Compressed LDA reduces the sample size to $m$, the rate $\mathcal{O}_{p}(m^{-1/2})$ in Theorem~\ref{thm:MisclassificationError} is expected. While the decay rate is typical, our theoretical approach is not. The main difficulty in analyzing Compressed LDA is dependency across $m$ compressed samples as (i)~they share the sample class mean $\overline{X}_{g}$, and (ii)~different rows of the compression matrix $Q^{g}$ can share the location of non-zero entries, and thus the same $\mathbf{x}_{i}^{g}$ may appear in \eqref{eq:CompressedSample} for different values of $j$. To overcome these difficulties, we use independence between the compression matrices $Q^g$ and original data matrices $X^g$ when bounding the difference between $\widehat{\Sigma}_{w, \mathbf{c}}$ and $\Sigma_{w}$. 
The detailed proof of Theorem 1, as well as supplementary Theorems and Lemmas, are presented in the Supplementary Materials.

Finally, while the scaling $\mathcal{O}_{p}(m^{-1/2})$ in Theorem~\ref{thm:MisclassificationError} is the same as what would be expected under sub-sampling (randomly selecting $m/2$ samples from each class and discarding the rest), we found that empirically compression offers two advantages: (i) it has the smaller misclassification error rate variance (see e.g. Figures~\ref{fig:Zip}-\ref{fig:MNIST}), which is likely due to using multiple $\mathbf{x}_{i}^{g}$ in forming each compressed sample; (ii)~it is more robust to violations of normality assumption in the original data as the summation within~\eqref{eq:CompressedSample} induces normality of compressed samples (see Figure~\ref{fig:SkinSegmentationPlots}).

\section{Extensions}
\subsection{Projected LDA}\label{sec:ProjectedLDA}

The Compressed LDA proposed in Section~\ref{sec:compLDA} proceeds by (i) forming a discriminant vector $\beta_{\mathbf{c}}$ based on compressed samples in~\eqref{eq:CompressedSample}; (ii) using $\beta_{\bc}$ and compressed within-class sample covariance matrix $\widehat \Sigma_{w, \mathbf{c}}$ in classification rule~\eqref{eq:ClassificationRule}. An alternative approach is to use step (i) only, project the original training data using $\beta_{\mathbf{c}}$ to form $\mathbf{z}_{i}^g = \beta_{\mathbf{c}}^{\top}\,\mathbf{x}_{i}^{g}\in \mathbb{R}$, and then apply LDA on the pairs $\{\mathbf{z}_i, y_i\}$, where now the samples $\mathbf{z}_i$ are one-dimensional scalars rather than $p$-dimensional vectors. Thus, the within-class variance of the projected data $\beta_{\mathbf{c}}^{\top}\widehat{\Sigma}_{w}\beta_{\mathbf{c}}$ is used in decision rule \eqref{eq:ClassificationRule} rather than $\beta_{\bc}^{\top}\widehat{\Sigma}_{w, \mathbf{c}}\beta_{\bc}$. We call this alternative approach Projected LDA. 
If the two classes have equal sample sizes, that is Assumption~\eqref{assump:EqualClass} holds, Compressed LDA and Projected LDA rules coincide as both will classify a new $\mathbf{x}$ according to
\begin{equation*}\label{eq:ClassificationEqual}
\argmin_{g=1,2}\{(\mathbf{x}-\overline{X}_{g})^\top \beta_{\bc}\}^2.
\end{equation*}
However, if $n_1\neq n_2$, the two methods will in general differ due to discrepancy between $\beta_{\mathbf{c}}^{\top}\widehat{\Sigma}_{w}\beta_{\mathbf{c}}$ and $\beta_{\bc}^{\top}\widehat{\Sigma}_{w, \mathbf{c}}\beta_{\bc}$.

The Projected LDA is analogous to the Fast Random Fisher Discriminant Analysis proposed in \citet{ye2017fast}: both use compression to form the discriminant vector $ \beta_{\bc}$, and then apply LDA on the projected values. The key difference between the two approaches is the compression scheme: \citet{ye2017fast} jointly compress both classes when forming $\beta_{\bc}$, whereas we propose separate class compression. We found that the latter is preferable, and Section~\ref{sec:Simulations} shows that Projected LDA has consistently better classification performance than the method of \citet{ye2017fast}. 

In terms of computational efficiency, Projected LDA described here and Compressed LDA of Section~\ref{sec:compLDA} are comparable - the main computational bottleneck of both is calculation of compressed $\widehat{\Sigma}_{w, \mathbf{c}}$. In terms of theoretical guarantees, since the methods coincide under Assumption~\ref{assump:EqualClass}, the results of Theorem~\ref{thm:MisclassificationError} apply to Projected LDA as well. In practice, the sample sizes are often not exactly equal, and thus in Section~\ref{sec:Simulations} we observe some difference in the empirical performance of Compressed LDA and Projected LDA. We found, however, that neither method has uniformly better classification performance over the other.

%





\subsection{Compressed QDA}\label{sec:compQDA}
The proposed compression scheme~\eqref{eq:CompressedSample} is applied separately to each class, and thus allowing us to assign classes to the compressed samples. This, in turn, allows us to compute class-specific compressed covariance matrices, which motivates us to consider an extension of Compressed LDA to the case of unequal class covariance structures.

Quadratic Discriminant Analaysis (QDA) \citep{hastie_elements_2009} is a generalization of LDA to the case of unequal class covariance matrices, which weakens Assumption~\ref{assump:Normality}.
\begin{assumption}\label{assump:QDANormality}
Conditional on class membership $g=1,2$, the samples $\mathbf{x}_{i}^{g} \in \mathbb{R}^{p}$ are i.i.d. $N(\mu_g\,,\, \Sigma_{w}^{g})$. 
\end{assumption}
Under Assumption~\ref{assump:QDANormality}, the Bayes decision rule classifies a new sample $\mathbf{x}\in \mathbb{R}^{p}$ by minimizing

\begin{equation}\label{eq:BayesQDA}
   \argmin_{g=1,2}\big\{ (\mathbf{x}-\mu_{g})^{\top}({\Sigma_{w}^{g}})^{-1}(\mathbf{x}-\mu_{g})+\log |\Sigma_{w}^{g}|-2\log(\pi_{g})\big\},
\end{equation}
where $|\Sigma_{w}^{g}|$ is the determinant of $\Sigma_{w}^{g}$. The QDA classification rule is the sample plug-in rule, where the population parameters $\mu_{g}$, $\Sigma_{w}^{g}$, and $\pi_{g}$ are replaced by their sample estimates $\overline{X}_{g},$ $\widehat{\Sigma}_{w}^{g}$, and $n_{g}/n$.

As our compression scheme proposed in~\eqref{eq:CompressedSample} is applied separately to each class, it can be used to form class-specific compressed covariance matrices.
\begin{definition}
The compressed sample covariance matrix for class $g=1,2$ is defined as
$$
\widehat{\Sigma}_{w, \mathbf{c}}^{g}:= \frac{1}{m_g}\sum_{j=1}^{m_g}(\mathbf{x}_{j, \mathbf{c}}^{g}-\overline{X}_{g})(\mathbf{x}_{j, \mathbf{c}}^{g}-\overline{X}_{g})^{\top}.
$$
\end{definition}
We define the Compressed QDA decision rule by substituting $\widehat{\Sigma}_{w, \mathbf{c}}^{g}$
instead of $\Sigma_{w}^{g}$ in~\eqref{eq:BayesQDA}, and $\overline{X}_{g}$, $n_{g}/n$ instead of $\mu_g$, $\pi_g$, respectively. 







\section{Simulation Studies}\label{sec:Simulations}
In this section we empirically evaluate the performance of the proposed compression methods on three publicly available datasets: Zip Code \citep{hastie_elements_2009}, MNIST \citep{lecun1998gradient} and Skin Segmentation \citep{bhatt2010skin}. For each dataset, we compare five linear classifiers: (L1)~Compressed LDA of Section~\ref{sec:compLDA}; (L2) Projected LDA of Section~\ref{sec:ProjectedLDA}; (L3) Fast Random Fisher Discriminant Analysis (FRF) of \citet{ye2017fast}; (L4) LDA trained on sub-sampled data drawn uniformly from both classes; and (L5) LDA trained on the full data (Full LDA). We also separately compare three quadratic classifiers: (Q1) Compressed QDA of Section~\ref{sec:compQDA}; (Q2) QDA trained on sub-sampled data drawn uniformly from both classes; and (Q3) QDA trained on the full data (Full QDA). 

For each method, we evaluate the out-of-sample misclassification error rate as a function of reduced number of training samples $m = m_1 + m_2$ (with $m=n$ for full methods L5 and Q3). To assess variability due to compression or sub-sampling, we use 100 replications for each value of $m$. Within each classifier, a small multiple of the identity matrix  $\gamma I_{p}$ is added to the corresponding estimate of the within-class covariance matrix $\Sigma_w$ for numerical stability. We use $\gamma = 10^{-4}$ for Zip Code and Skin Segmentation data, and $\gamma = 10^{-3}$ for the MNIST data as it has a much larger number of features $p$ compared to other datasets, and thus requires stronger regularization. We use $s=0.01$ for Zip Code and MNIST datasets, and $s=10^{-3}$ for the Skin Segmentation dataset as the latter has considerably larger sample size $n$; thus for all datasets $s=O(n^{-1/2})$.

We also compare the execution times of forming the compressed within-class covariance matrix $\widehat \Sigma_{w,\bc}$ and full within-class covariance matrix $\widehat \Sigma_w$. For compression, we consider the time required to both compress the data via $Q^g$ and to form $\widehat \Sigma_{w, \bc}$. The timing results are reported using a Linux Machine with Intel Xeon E5-2690 with 2.90 GHz. 

\subsection{ZIP Code Data}\label{sec:ZipCodeData}
The Zip Code Data \citep{hastie_elements_2009} has $n=7,291$ training samples with $p=256$ features. The samples are images of handwritten digits for zip codes, and each feature corresponds to a normalized gray-scale pixel of an image. The original data has ten classes, each corresponding to a digit from 0 to 9, which we merge into two classes of even and odd digits. The classes are well-balanced, with $48\%$ to $52\%$ split between the class $1$ odd digits and class $2$ even digits. The corresponding test data has $n=2,007$ samples. 

The top of Figure~\ref{fig:Zip} displays the misclassification error rates of (L1)-(L5) across $100$ independent trials for each value of $m$. As expected, the performance of all methods improves with the increase in compression dimension $m$. Both Compressed LDA and Projected LDA have better classification performance compared to FRF and sub-sampled LDA. For example, when $m=500$, Compressed LDA has a mean misclassification error rate of $12.60\%$ (se $0.08\%$), and Projected LDA has mean error rate $12.73\%$ (se $0.08\%$). In contrast, FRF has a mean rate of $13.84\%$ (se $0.08\%)$, and sub-sampling has mean rate $15.31\%$ (se $0.13\%$).  Compressed and Projected LDA have similar error rates due to the balanced class sizes in this dataset, see Section~\ref{sec:ProjectedLDA}.

\begin{figure}[!t]
\centering
\includegraphics[width=8cm]{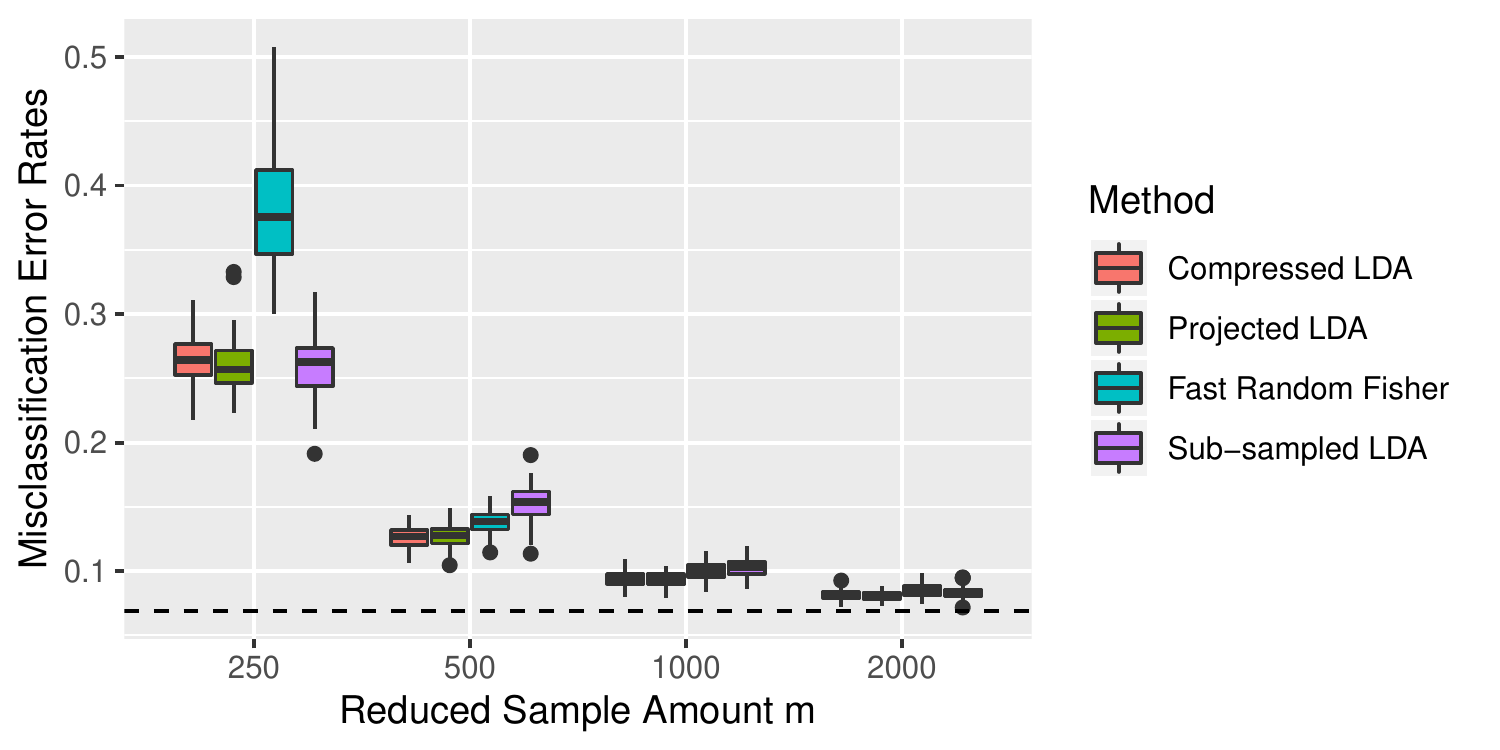}
\includegraphics[width=8cm]{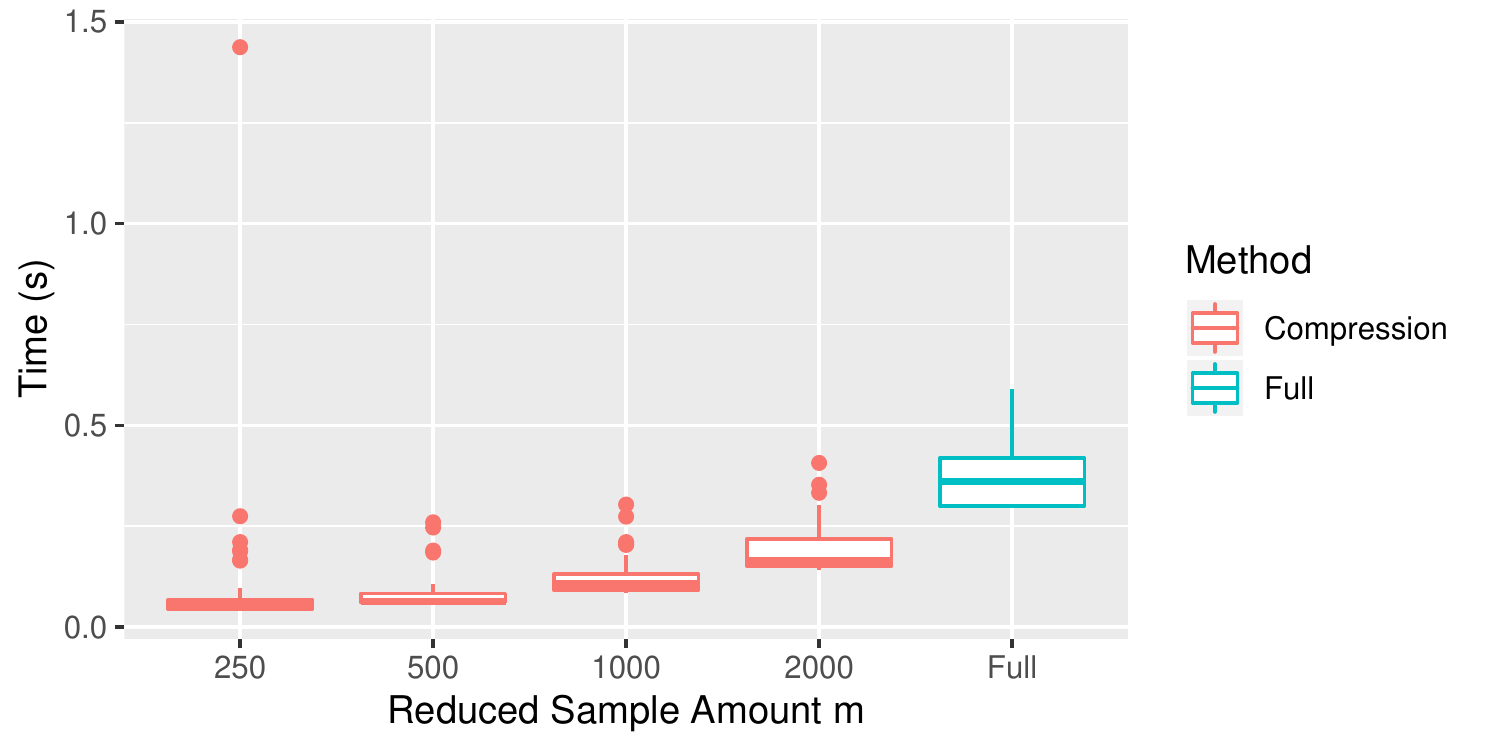}
\caption{Zip Code Data. \textbf{Left:} Misclassification error rates across 100 replications for each value of $m$ with $s=0.01$ and $\gamma=10^{-4}$.  The dashed line represents the $6.88\%$ error rate of Full LDA.
\textbf{Right:} The execution times for 100 independent compressed and full covariance formations. }
\label{fig:Zip}
\end{figure} 

Compressed and Projected LDA have the lowest mean error rates and standard errors across all values of $m$. Sub-sampling has the highest mean error rates for $m \geq 500$, which is likely because pixel values for images of handwritten digits are not normally distributed. Unexpected to us, FRF has the highest error rates for $m=250$ despite using compression. We suspect this is due to its joint compression of both classes (rather than separate class compression used by our methods), which likely leads to higher variance in the estimated discriminant vector when $m$ is relatively small. When $m\geq 500$, the error rates of FRF are better than sub-sampling, but still worse than the proposed approaches.

The bottom of Figure~\ref{fig:Zip} compares the execution times of forming compressed and full within-class covariance matrices, where the execution time for compression includes both formation of compressed samples in~\eqref{eq:CompressedSample} and calculation of $\widehat \Sigma_{w, \bc}$. As expected, compression is significantly faster. For instance, when $m=2,000$, the compression takes on average $0.19$ seconds (se $0.01$ s), while the construction of full covariance matrixtakes on average $0.36$ seconds (se $0.01$ s). 

Figure~\ref{fig:QDAZip} displays the misclassification error rates of (Q1)-(Q3). Compressed QDA has uniformly lower mean error rates and lower variance than QDA on sub-sampled data for the same values of $m$. For instance, when $m = 500$, Compressed QDA has a mean error rate of $12.22\%$ (se $0.08\%$) while sub-sampled QDA has the mean error rate of $19.27\%$ (se $0.14\%$). For $m\geq 2,000$, the misclassification error rate of Compressed QDA matches that of Full QDA.

\begin{figure}[!t] 
\centering 
\includegraphics[width=10cm]{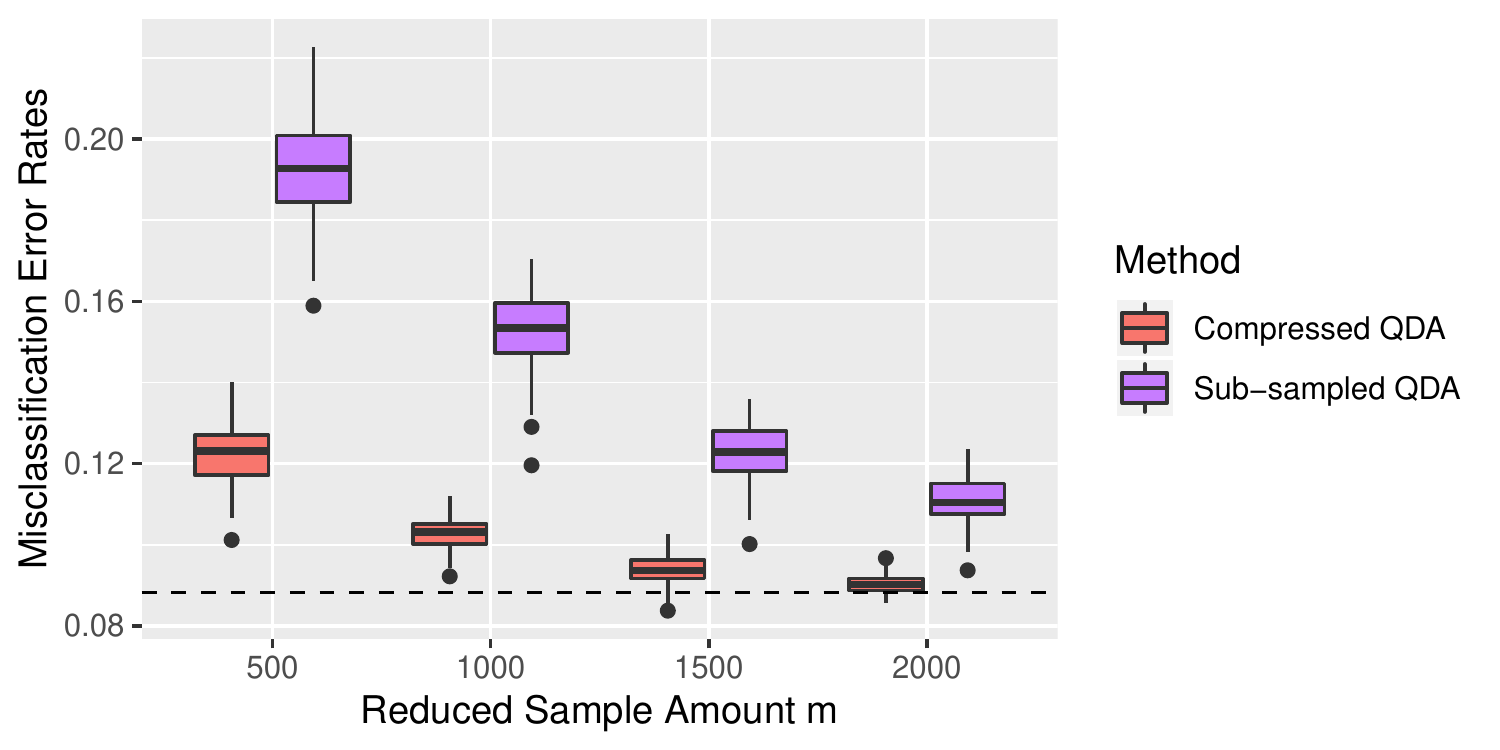}
\caption{Zip Code Data. Misclassification error rates of compressed and sub-sampled QDA across 100 replications for each value of $m$ with $s=0.01$ and $\gamma =10^{-3}$. The dashed line represents the $8.82\%$ error rate of Full QDA.} 
\label{fig:QDAZip} 
\end{figure}

\subsection{MNIST Data}

The MNIST Data \citep{hastie_elements_2009} has $n=60,000$ training samples with $p=784$ features. The samples are pictures of handwritten digits, and  each feature corresponds to a normalized grayscale pixel for an image. The original data has ten classes, each corresponding to a digit from 0 to 9, which we merge into two classes of even and odd digits. The classes are well-balanced with a $51\%$ to $49\%$ split between the class $1$ odd digits and class $2$ even digits. The test data has $n=10,000$ samples. 

The top of Figure~\ref{fig:MNIST} shows the misclassification error rates of the linear methods across $100$ independent trials for each value of $m$. As with the Zip Code data, both Compressed LDA and Projected LDA have the lowest misclassification error rates compared to FRF and sub-sampled LDA. For instance, when $m=2,000$, the mean error rate for Compressed LDA is $13.93\%$ (se $0.04\%$), and the mean error rate for Projected LDA is $13.98$ (se $0.04\%$). In contrast, FRF has mean rate $15.71\%$ (se $0.05\%)$, and sub-sampled LDA has mean rate $16.05\%$ (se $0.05\%$). As with the Zip Code data, Compressed and Projected LDA have similar rates due to the balanced class sizes, see Section~\ref{sec:ProjectedLDA}. Unlike the Zip Code data, FRF performs comparable to sub-sampling even for larger values of $m$. This suggests that joint class compression leads to sub-optimal classification performance compared to proposed separate class compression, and the difference is particularly striking when the number of features $p$ is large.

The bottom of Figure~\ref{fig:MNIST} compares the execution times of forming compressed and full within-class covariance matrices. As expected, compression is considerably faster. Even when $m=10,000$, the mean time for compression ($9.31$ seconds, se $1.29$) is significantly smaller than the time of forming $\widehat \Sigma_w$ on the full data ($23.53$ seconds, se $ 2.29$). 

\begin{figure}[!t]
  \centering
   \includegraphics[width=8cm]{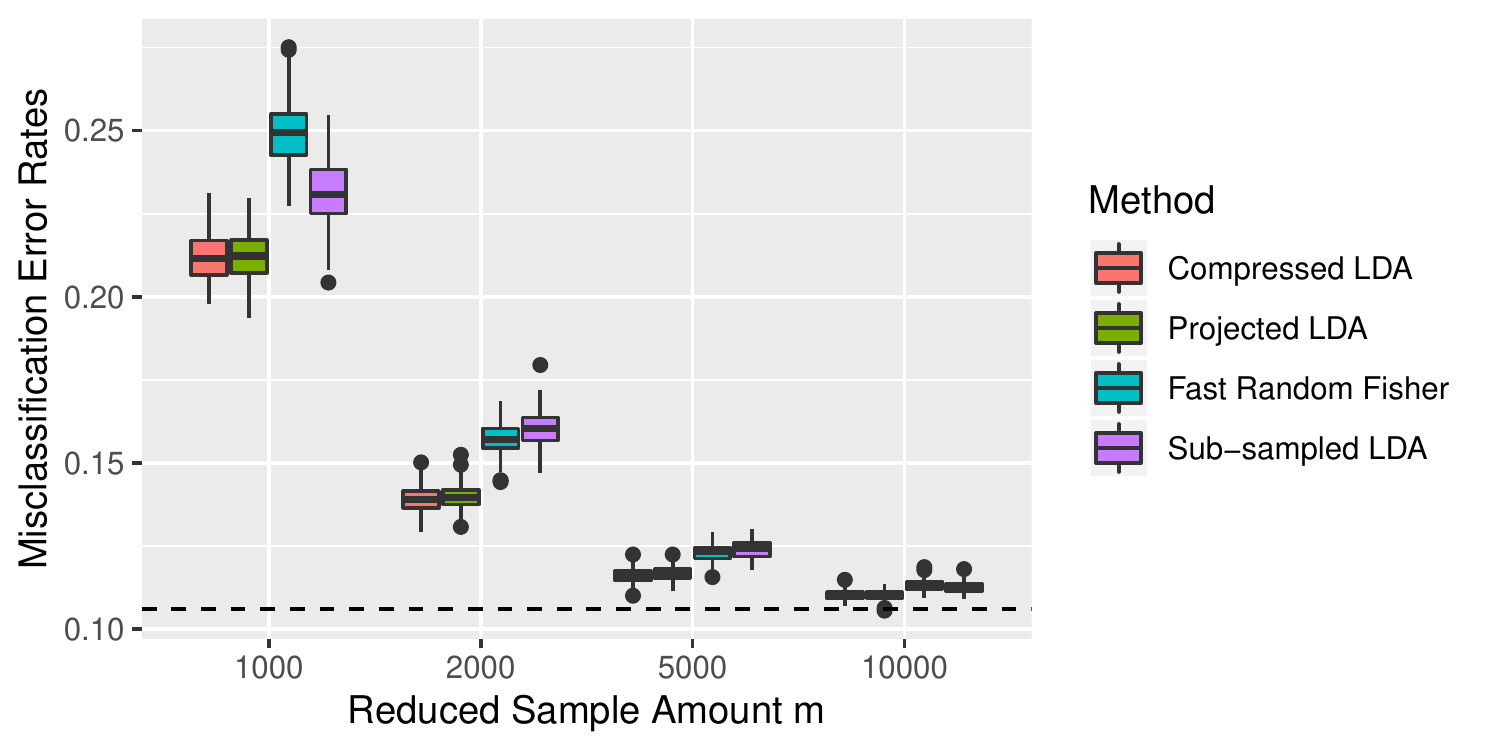}
   \includegraphics[width=8cm]{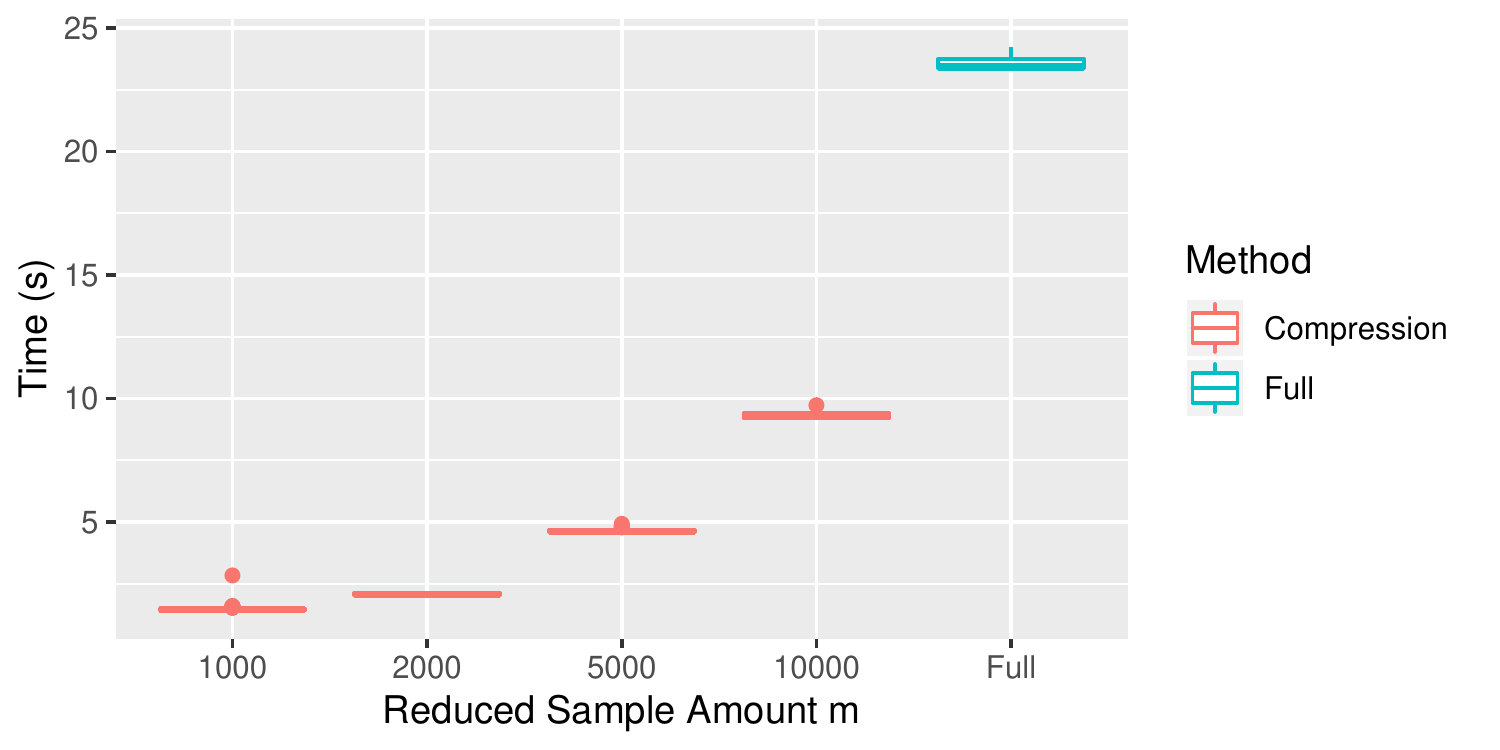}
  \caption{MNIST Data. \textbf{Left:} Misclassification error rates across 100 replications for each value of $m$ with $s=0.01$ and $\gamma =10^{-3}$. The dashed line represents the $10.60\%$ misclassification error rate of Full LDA.
 \textbf{Right:} The execution times for 100 independent compressed and full covariance formations.}
\label{fig:MNIST}
\end{figure}

\begin{figure}[!t]
\centering 
\includegraphics[width=10cm]{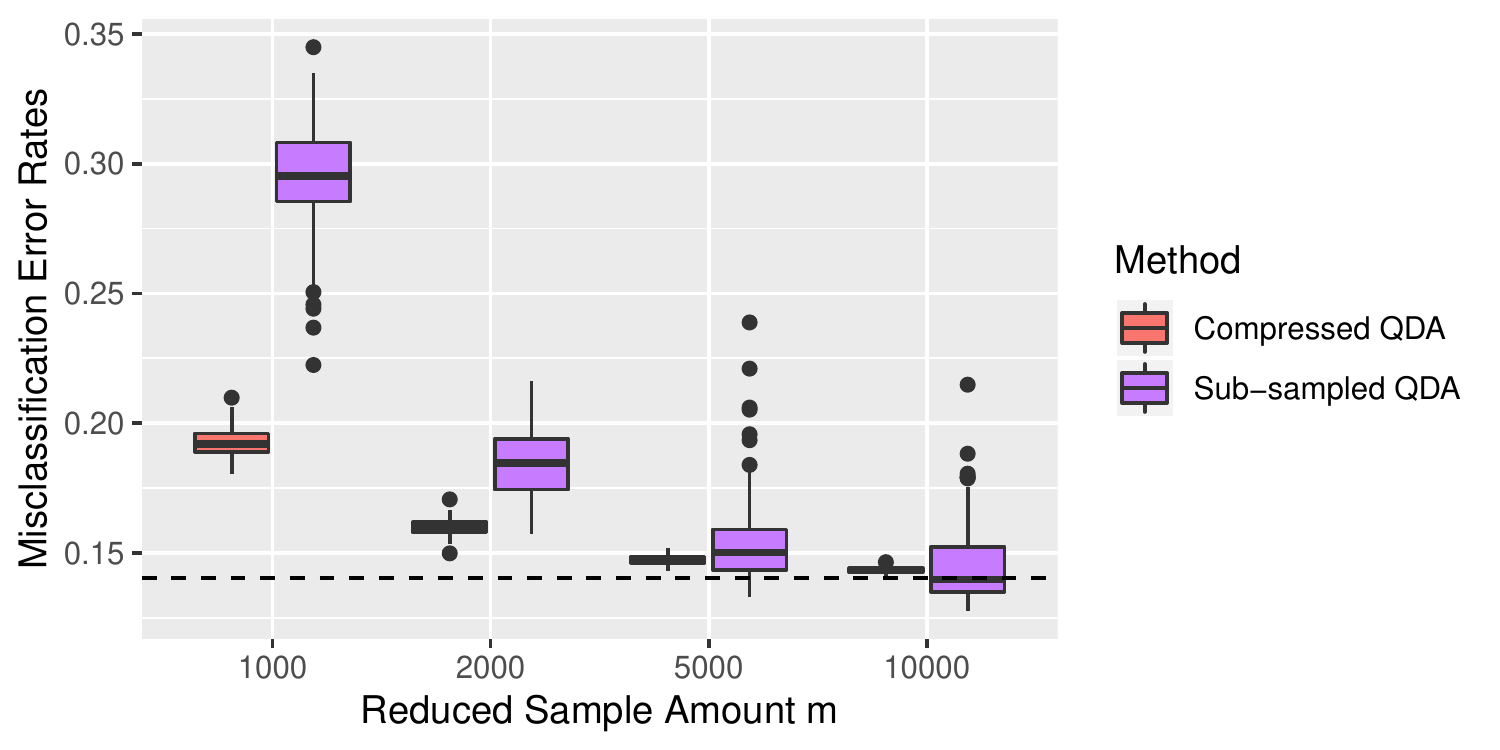}
\caption{MNIST Data. Misclassification error rates of compressed and sub-sampled QDA across 100 replications for each value of $m$ with $s=0.01$ and $\gamma =10^{-3}$. The dashed line represents the $14.04\%$ error rate of Full QDA.} 
\label{fig:QDA_MNIST} 
\end{figure}

Figure~\ref{fig:QDA_MNIST} shows the misclassification error rates of the quadratic methods. Compressed QDA has uniformly better performance than sub-sampling, it has both lower mean error rates and lower variances. For example, when $m=1,000$, Compressed QDA has mean error rate $19.24\%$ (se $0.06\%$) while sub-sampled QDA has mean error $29.42\%$ (se $0.21\%$).

\subsection{Skin Segmentation Data}\label{sec:SkinSegmentation}
The Skin Segmentation Data \citep{bhatt2010skin}
has $n=245,057$ samples with $p=3$ features. The features are Red, Blue, and Green pixel values for randomly sampled image pixels. The goal is to learn which colors represent skin, and subsequently classify those pixels as corresponding to skin or not. Unlike the Zip Code and MNIST datasets, here the classes are unbalanced, with $21\%$ (skin) to $79\%$ (not skin) split. We select $90\%$ of the data from each class for training, and use the remaining 10\% for testing.

The top of Figure~\ref{fig:SkinErrorsAndTimes} displays the misclassification error rates of the linear methods across $100$ independent trials for each value of $m$. Compressed LDA, Projected LDA, and FRF all have superior classification performance over sub-sampled LDA, especially in terms of variance for the same value of $m$. For instance, when $m=25$, Compressed LDA has an average error rate of $7.42\%$ (se $0.09\%$), with $7.57\%$ (se $0.09\%$) for Projected LDA, and $7.38\%$ (se $0.09\%$) for FRF. In contrast, sub-sampled LDA has error $8.78\%$ (se $0.40\%$). Unlike the Zip Code and MNIST datasets, FRF performs comparably to the proposed approaches, which supports our previous conjecture that the difference between joint compression and separate class compression is more pronounced for larger values of $p$. 
The bottom of Figure~\ref{fig:SkinErrorsAndTimes} displays the corresponding error rates for the quadratic methods. While the mean error rates between Compressed QDA and sub-sampled QDA are similar, Compressed QDA has much smaller variance, which is consistent with results we observed for other datasets. 

\begin{figure}[!t]
\centering 
\includegraphics[width=8cm]{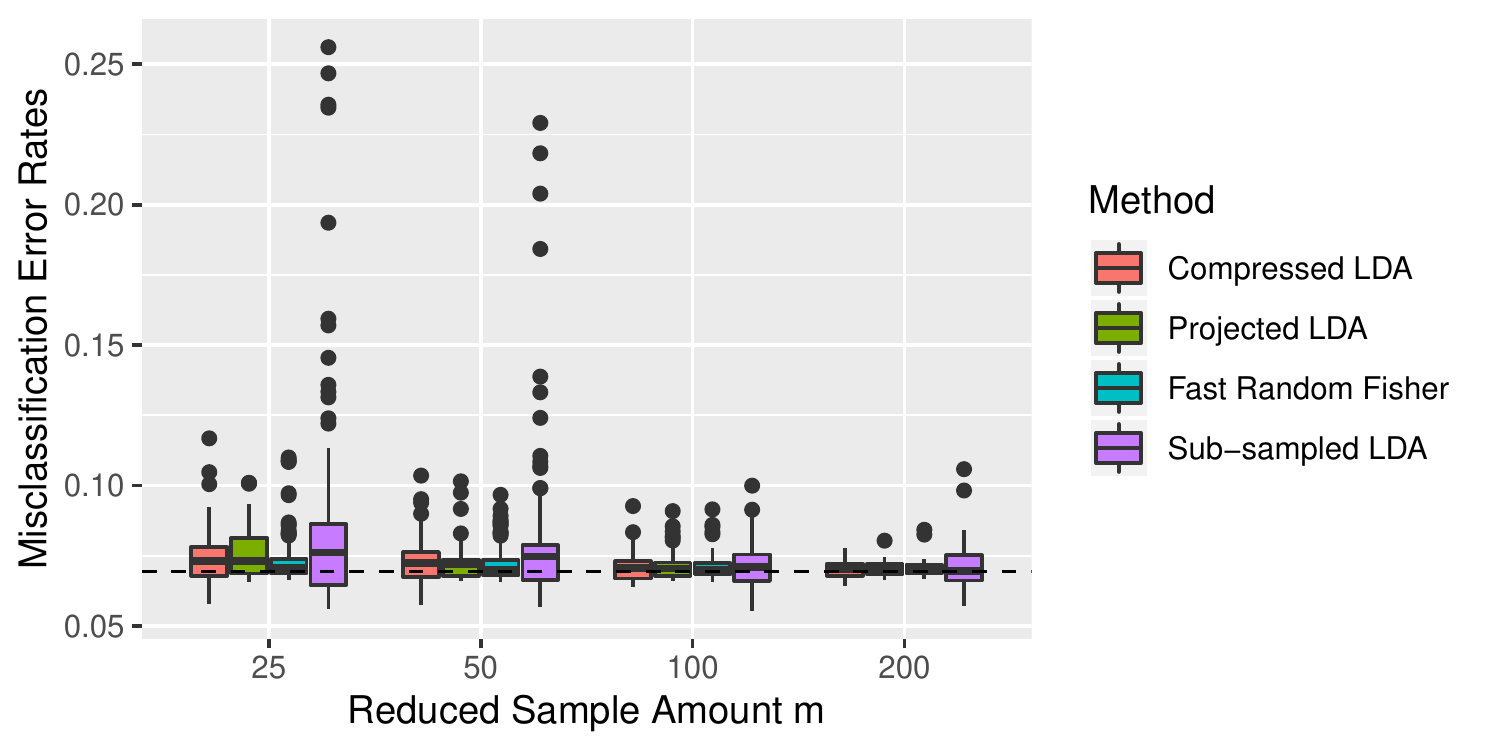}
\includegraphics[width=8cm]{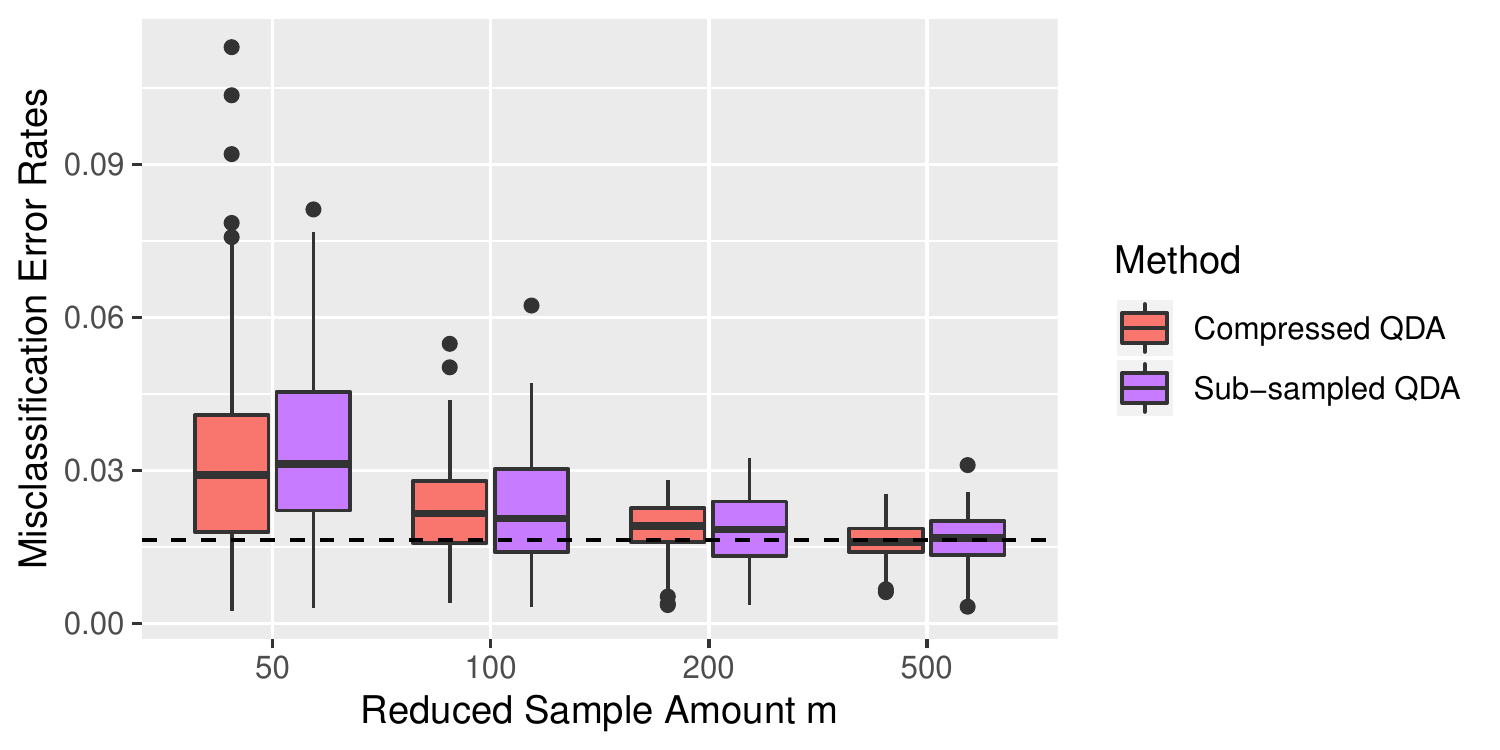}
\caption{Skin Segmentation Data, misclassification error rates across 100 replications for each vale of $m$. \textbf{Left:} Linear classification methods with $s=10^{-3}$ and $\gamma =10^{-4}$. The dashed line represents the $6.93\%$ error rate of Full LDA. 
\textbf{Right:}~Qadratic classification methods with $s=10^{-3}$ and $\gamma =10^{-4}$. The dashed line represents the $1.64\%$ error rate of Full QDA. 
}
\label{fig:SkinErrorsAndTimes} 
\end{figure}

The Skin Segmentation Data only has $p=3$ features, and thus one may ask whether the compression is really necessary since it doesn't offer significant computational advantages for small values of $p$. We found, however, that compression still allows to use much smaller number of samples to obtain good predictive accuracy, as Compressed LDA reaches the Full LDA error rate of $6.93\%$ at only $m=100$. Furthermore, our main reason for including this dataset as an example is to illustrate how compression can induce normality in the compressed samples when the normality for original samples does not hold. The top of Figure~\ref{fig:SkinSegmentationPlots} shows the first two principal components of $5,000$ original training samples, whereas the bottom of Figure~\ref{fig:SkinSegmentationPlots} shows the first two principal components of $5,000$ compressed samples. The original training samples clearly are not normally distributed as the main directions of variation display non-linear class separation. In contrast, each class of compressed data has an elliptical shape suggesting the normal distribution and a linear classification boundary. Thus, Compressed LDA is more robust to the assumption of normality than sub-sampling. For the Skin Segmentation Data, this leads to Compressed LDA having slightly lower mean misclassification error rates compared to sub-sampling, and significantly smaller error variances across the replications.

\begin{figure}[!t] 
\centering 
\includegraphics[width=10cm]{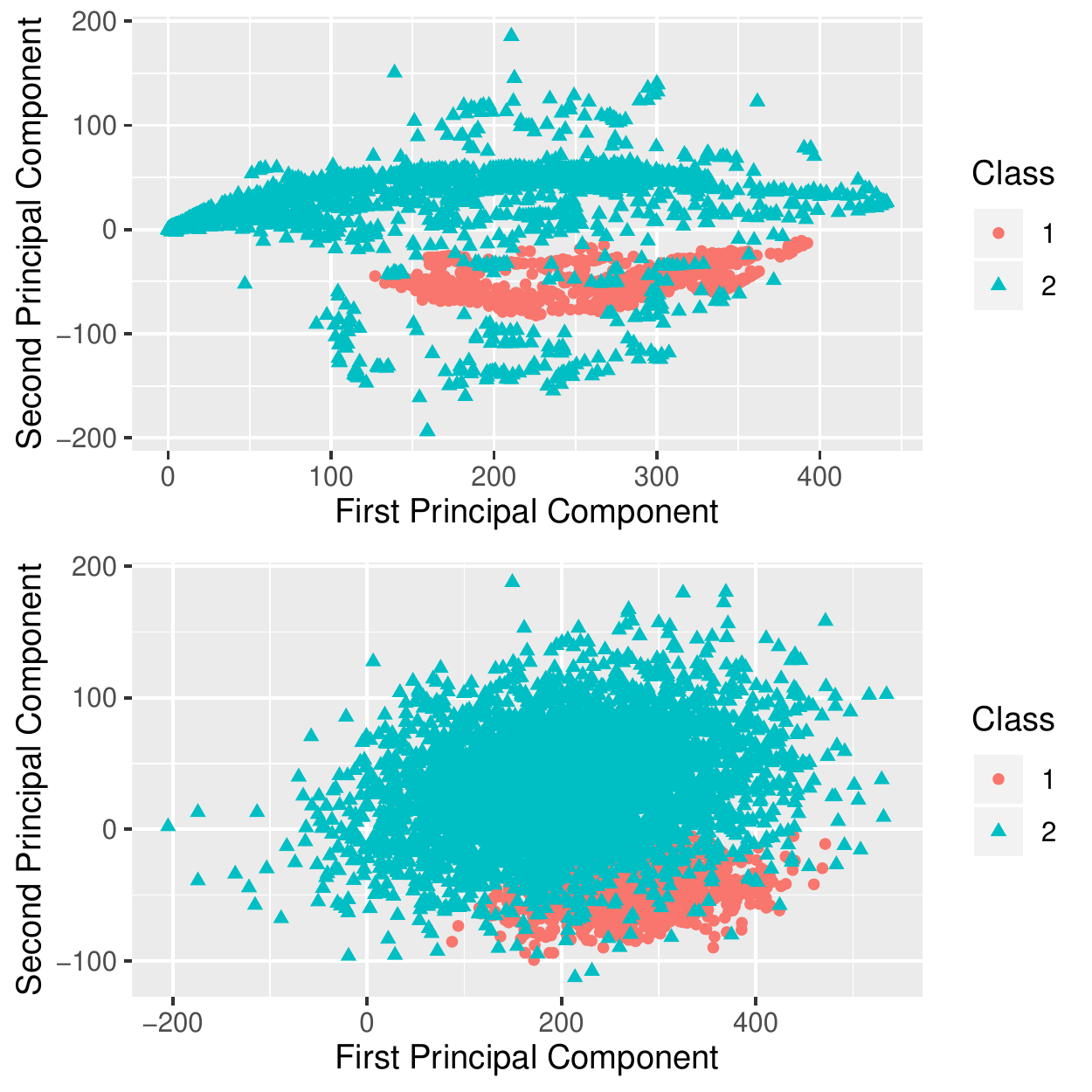}
\caption{
Skin Segmentation Data, the two classes are separated by both shape and color. \textbf{Top:} First two principal components based on $5,000$ training samples. \textbf{Bottom:} First two principal components based on $5,000$ compressed samples with $s=0.001$.
} 
\label{fig:SkinSegmentationPlots} 
\end{figure}

\begin{figure}[!t]
\centering
\includegraphics[width=4in]{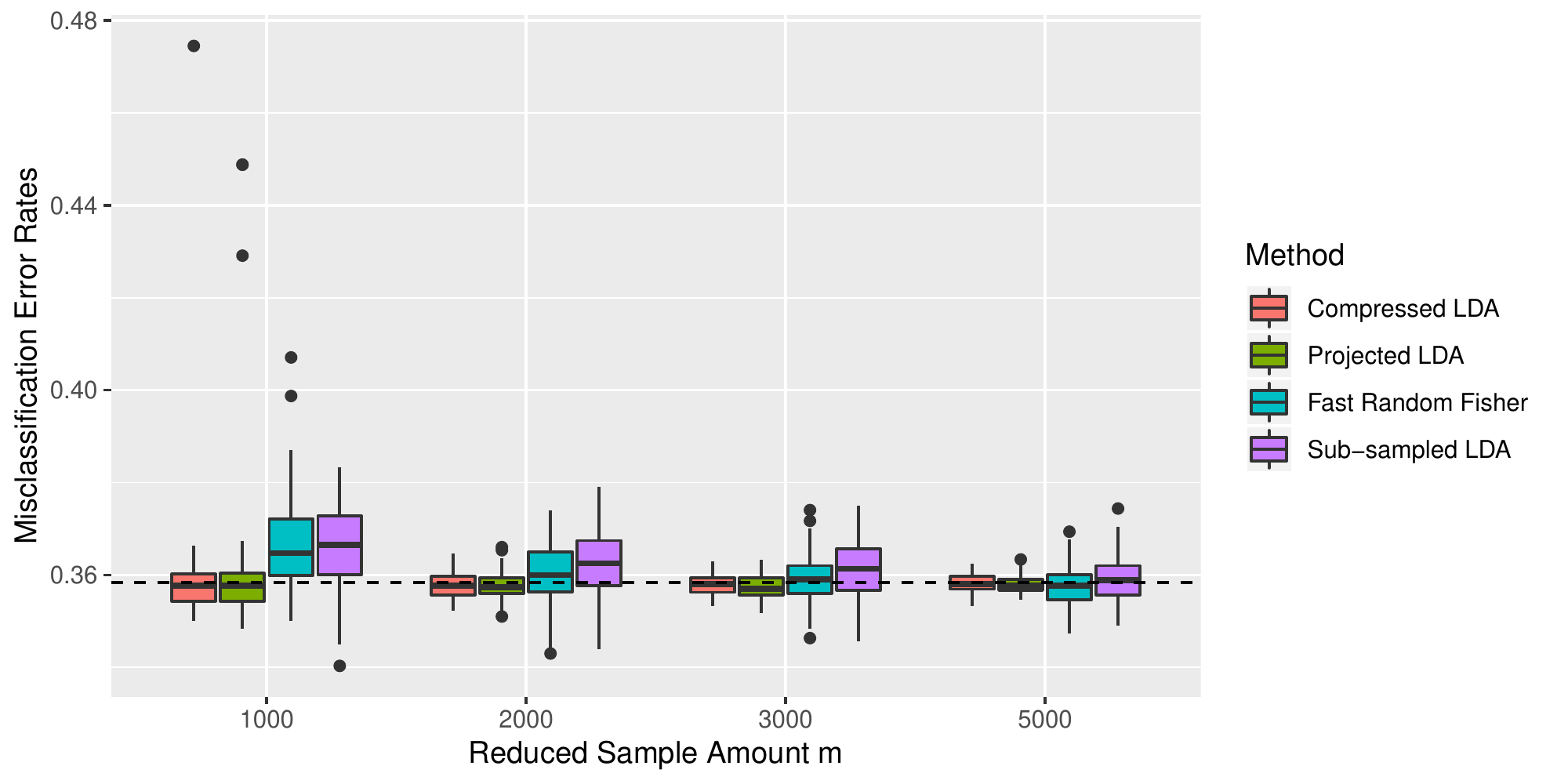} 
\caption{Eye State Data.  Misclassification error rates across 100 replications of compressed LDA, projected LDA, FRF, and sub-sampled LDA across different sample amounts $m$ for $s=0.01$ and $\gamma=10^{-3}$. The dashed line represents the $35.84\%$ error rate of Full LDA. }
\label{fig:EyeState}
\end{figure} 

\subsection{Eye State Data}
We consider the Eye State data \citep{rosler2013first}. This data set has $n = 14,980$ samples with $p = 14$ features corresponding to electroencephalography (EEG) measurements. The goal is to predict whether the eye state is open or closed during the time of the EEG measurements. The data set is almost balanced, with a $44.88\%$ (open) to $55.12\%$ (closed) split. Figure~\ref{fig:EyeState} displays the misclassification error rates across $100$ independent iterations of the linear methods. 

Both compressed and sub-sampled LDA have uniformly lower error rates compared to FRF and sub-sampling. For example, at $m=1,000$, compressed LDA has mean error rate $35.88\%$ (se $0.12\%$), and projected LDA has mean error rate $35.97\%$ (se $0.15\%$). This is compared with FRF which has a mean error rate of $36.63\%$ (se $ 0.10\%$), and sub-sampled LDA which has a mean error rate of $36.61\%$ (se $0.09\%$).

\section{Discussion}
We propose a sample reduction scheme for discriminant analysis through compression. The advantage of compression over sub-sampling is illustrated in Section~\ref{sec:Simulations}, where the proposed Compressed LDA consistently has better classification performance than LDA trained on sub-sampled data. 
The compression scheme is further extended to Projected LDA and Compressed QDA, which again show superior predictive accuracy compared to the same classifiers trained on sub-sampled data.

There are several directions of future research that could be pursued.
First, while we only considered binary classification, our approach can be extended to the multi-class setting by applying compression~\eqref{eq:CompressedSample} to all $G$ classes. Secondly, given our results on compressing in the number of samples, and existing results on compressing in the number of features \citep{li2019one,tu2014making}, it would be of interest to simultaneously consider both compression schemes within discriminant analysis. Finally, here we focused on linear and quadratic classification rules which may be too restrictive. Exploring compression within the kernel discriminant analysis framework \citep{mika1999fisher} will allow for more flexible non-linear classification boundaries.

\section*{Acknowledgments}
This work was supported in part by NSF-DMS 1712943.

\begin{appendix}


\section{Proofs of Theoretical Results}\label{sec:Proofs}

This section contains a proof of Theorem \ref{thm:MisclassificationError} along with supplemental Theorems and Lemmas. In the following $C$ denotes an absolute constant which may change from line to line. If multiple constants appear in the same expression, $C_1$, $C_2$, etc. will be used to differentiate them.

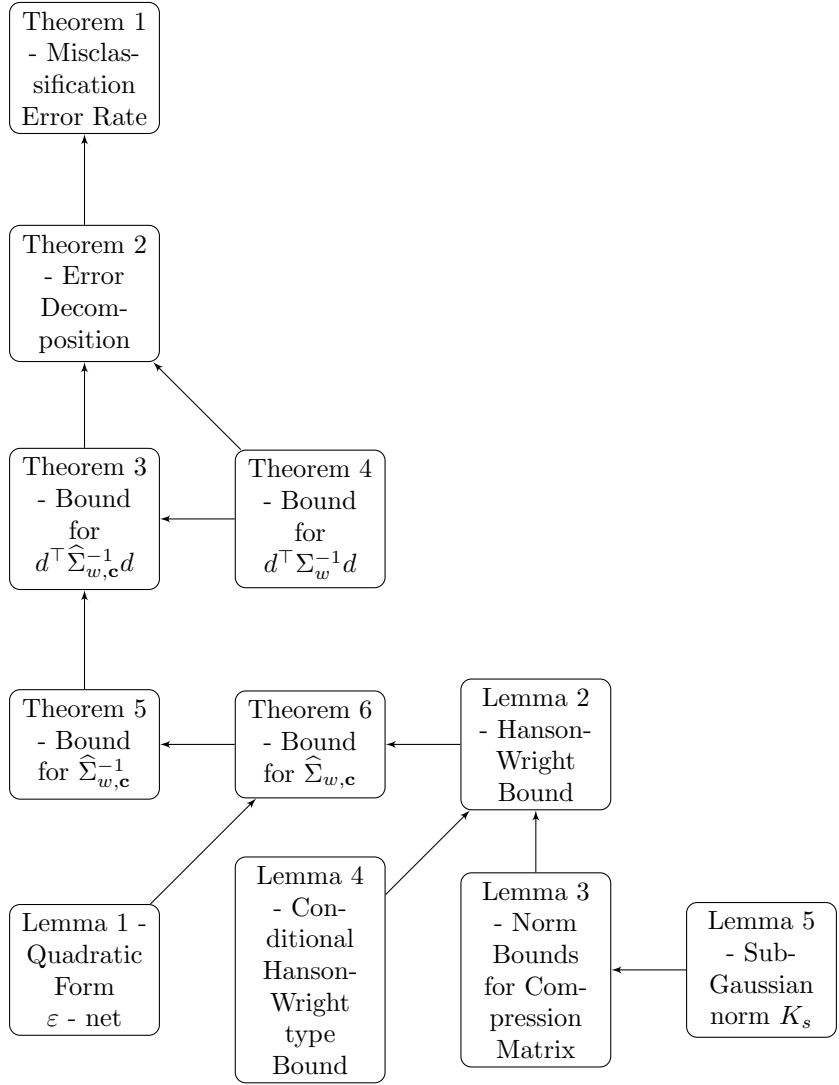
\begin{figure}[ht]
\begin{center}
\begin{tikzpicture}[node distance=3cm, auto]
 
\node [block] (init) {Theorem \ref{thm:MisclassificationError} - Misclassification Error Rate} ;
\node [block, below of=init](block-2) {Theorem~\ref{thm:ErrorDecomp} - Error Decomposition} ;
 \node [block, below of=block-2] (block-3) {Theorem~\ref{thm:CovAndMean} - Bound for $d^{\top}\widehat{\Sigma}_{w, \mathbf{c}}^{-1}d$} ;
 \node [block, right of = block-3] (block-4) {Theorem~\ref{thm:MeanInvCovariance} - Bound for $d^\top \Sigma_{w}^{-1}d$} ;
 \node [block, below of = block-3] (block-5){Theorem~\ref{thm:InvCov} - Bound for $\widehat{\Sigma}_{w, \mathbf{c}}^{-1}$};
 \node [block, right of = block-5] (block-6){Theorem~\ref{thm:SigmaBound} - Bound for $\widehat{\Sigma}_{w, \mathbf{c}}$};
 \node [block, below of = block-5] (block-7) {Lemma~\ref{lem:NetQuadForm} - Quadratic Form $\varepsilon$ - net};
 \node [block, right of = block-6](block-8){Lemma~\ref{lem:UnconditionalHW} -  Hanson-Wright Bound};
  \node [block, below of = block-8](block-10){Lemma~\ref{lem:NormBounds} - Norm Bounds for Compression Matrix};
   \node [block, left of = block-10](block-9){Lemma~\ref{lem:ConditionalHW} - Conditional Hanson-Wright type Bound};
   \node [block, right of = block-10](block-11){Lemma~\ref{lem:sub-GaussianNorm} - Sub-Gaussian norm $K_{s}$};

\path [line] (block-2) -- (init) ;
\path [line] (block-3) -- (block-2) ;
\path [line] (block-4) -- (block-2) ;
\path [line] (block-4) -- (block-3) ;
\path [line] (block-5) -- (block-3) ;
\path [line] (block-6) -- (block-5) ;
\path [line] (block-7) -- (block-6);
\path [line] (block-8) -- (block-6);
\path [line] (block-9) -- (block-8);
\path [line] (block-10) -- (block-8);
\path [line] (block-11) -- (block-10);
\end{tikzpicture}
\caption{Proof chart for Theorem \ref{thm:MisclassificationError}. }
\end{center}
\end{figure}

We make the following assumption which is useful for simplifying expressions in the theory. 
\begin{assumption}\label{assump:Sizeofm}
The number of compressed samples $m$ is large enough so that $\log(\eta^{-1})/m\leq 1$. Additionally, the number of original training samples $n$ is large enough so that $\log(\eta^{-1})\leq \sqrt{n}.$
\end{assumption}
\begin{remark}
Assumption~\ref{assump:Sizeofm} is mild. For instance, if $\eta= 10^{-10}$, then $m$ must be at least $24$, and $n$ must be at least $531$. If $\eta=10^{-2},$ then $m$ muust be at least $5$, and $n$ must be at least $22$
\end{remark}

\newpage

\begin{proof}[Proof of Theorem~\ref{thm:MisclassificationError}]
By Theorem~\ref{thm:ErrorDecomp}, the compressed LDA misclassification error rate $R_{\mathbf{c}}$ has the form
$$
R_{\textbf{c}}=f(\varepsilon_1^1,\varepsilon^2_1,\varepsilon_2) =\frac{1}{2}\sum_{g=1}^{2} 
\Phi \bigg( \frac{\varepsilon_{1}^{g}-\delta^\top\Sigma_{w}^{-1}\delta
}{\sqrt{
\varepsilon_{2}+
\delta^{\top}\Sigma_{w}^{-1}\delta 
}}\bigg),
$$
where $\varepsilon_1^g$ and $\varepsilon_2$ are defined in Theorem~\ref{thm:ErrorDecomp}. Let $\varepsilon = (\varepsilon_{1}^{1}, \varepsilon_{1}^{2}, \varepsilon_{2})$. Taking the first-order Taylor expansion of $f$ centered at $0$ gives
$$
R_{\textbf{c}} = f(\varepsilon)= \Phi(-\sqrt{\delta^\top \Sigma_{w}^{-1}\delta})+\nabla f(0)^\top \varepsilon+o_{p}(\|\varepsilon\|_{2})=R_{\text{opt}}+\nabla f(0)^\top \varepsilon+o_{p}(\|\varepsilon\|_{2}).
$$
Plugging this expansion into $|R_{\mathbf{c}}-R_{\text{opt}}|$ gives
\begin{align*}
    | R_{\mathbf{c}} - R_{\text{opt}}|&=
    \bigg|
   \Phi(-\sqrt{\delta^\top \Sigma_{w}^{-1}\delta})+
   \nabla f(0)^\top \varepsilon\,+o_{p}(\|\varepsilon\|_{2})-R_{\text{opt}}
\bigg|\\
  &\leq\bigg|
   R_{\text{opt}}+\nabla f(0)^\top \varepsilon-R_{\text{opt}}\bigg|+o_{p}(\|\varepsilon\|_{2})\\
   &=\bigg| \nabla f(0)^\top \varepsilon
\bigg|+o_{p}(\|\varepsilon\|_{2})\\
&\leq   C \|\nabla f(0)\|_{2}\,\|\varepsilon\|_{2},
\end{align*}
where we absorbed the lower-order $o_{p}(\|\varepsilon\|_{2})$ into the absolute constant $C>0$. 

We now compute $\|\nabla f(0)\|_{2}.$ The partial derivatives are
$$
\frac{\partial f}{\partial \varepsilon_{1}^{g}}(0)
=\frac{1}{2}\phi\bigg( 
\frac{-\delta^\top \Sigma_{w}^{-1}\delta}{\sqrt{\delta^\top \Sigma_{w}^{-1}\delta}}
\bigg)\bigg[\frac{1}{\sqrt{\delta^\top \Sigma_{w}^{-1}\delta}} \bigg]= \frac{\phi(\sqrt{\delta^\top \Sigma_{w}^{-1}\delta})}{2 \sqrt{\delta^\top \Sigma_w^{-1} \delta}}
$$
and
$$
\frac{\partial f}{\partial \varepsilon_{2}}
(0)
=-\frac{1}{4}
\phi\bigg( 
\frac{-\delta^\top \Sigma_{w}^{-1}\delta}{\sqrt{\delta^\top \Sigma_{w}^{-1}\delta}}
\bigg)
\bigg[ 
\frac{-\delta^\top \Sigma_{w}^{-1}\delta}{(\delta^\top\Sigma_{w}^{-1}\delta)^{3/2}}
\bigg]=\frac{\phi(\sqrt{\delta^\top \Sigma_{w}^{-1} \delta})}{4 \sqrt{\delta^\top \Sigma_{w}^{-1} \delta}},
$$
where $\phi$ denotes the standard normal density. It follows that
\begin{align*}
    \|\nabla f(0)\|_{2}=\frac{\phi(-\sqrt{\delta^\top \Sigma_{w}^{-1} \delta})}{2 \sqrt{\delta^\top \Sigma_{w}^{-1} \delta}}\|\begin{pmatrix} 1&1&1/2\end{pmatrix}\|_{2}=
    \frac{3\,\phi(\sqrt{\delta^\top \Sigma_{w}^{-1} \delta})}{4 \sqrt{\delta^\top \Sigma_{w}^{-1} \delta}}.
\end{align*}

We now focus on bounding the error term $\|\varepsilon\|_{2}$. We have
\begin{align*}
    \|\varepsilon\|_{2}\leq \|\varepsilon\|_{1}&=|\varepsilon_{1}^{1}|+|\varepsilon_1^2|+|\varepsilon_2|.
\end{align*}
Applying Theorem~\ref{thm:ErrorDecomp} proves that with probability at least $1-\eta:$
\begin{align*}
|\varepsilon_{1}^{g}|&\leq C\, K_{s}^{2}\,(\|\Sigma_{w}^{-1/2}\delta\|_{2}^{2}+\|\Sigma_{w}^{-1/2}\delta\|_{2})\sqrt{\frac{\log(\eta^{-1})+p}{m}}\\
|\varepsilon_{2}|&\leq C\, K_{s}^{2}\,(\|\Sigma_{w}^{-1/2}\delta\|_{2}^{2}+\|\Sigma_{w}^{-1/2}\delta\|_{2})\sqrt{\frac{\log(\eta^{-1})+p}{m}}.
\end{align*}
It follows that with probability at least $1-\eta:$
\begin{align*}
   |R_{\textbf{c}} -R_{\text{opt}}|&\leq  C \frac{\,\phi(\sqrt{\delta^\top \Sigma_{w}^{-1} \delta})}{ \sqrt{\delta^\top \Sigma_{w}^{-1} \delta}}\, K_{s}^{2}\,(\|\Sigma_{w}^{-1/2}\delta\|_{2}^{2}+\|\Sigma_{w}^{-1/2}\delta\|_{2})\sqrt{\frac{\log(\eta^{-1})+p}{m}}\\
   &\leq 
    C \,\phi(\sqrt{\delta^\top \Sigma_{w}^{-1} \delta})\, K_{s}^{2}\,(\sqrt{\delta^\top \Sigma_{w}^{-1}\delta}+1)\sqrt{\frac{\log(\eta^{-1})+p}{m}}.
\end{align*}
This proves the Theorem.\qedhere
\end{proof}

\vspace{.2in}

\begin{theorem}\label{thm:ErrorDecomp}
 Let $R_{\mathbf{c}}$ be the misclassification error rate \eqref{eq:COSerror} of the compressed LDA decision rule.
Then $R_\mathbf{c}$ has the form
\begin{align*}
R_{\mathbf{c}}=
\frac{1}{2}\sum_{g=1}^{2} 
\Phi \bigg( \frac{\varepsilon_{1}^{g}-\delta^\top\Sigma_{w}^{-1}\delta
}{\sqrt{
\varepsilon_{2}+
\delta^{\top}\Sigma_{w}^{-1}\delta 
}}\bigg),\\
\end{align*}
where
\begin{align*}
\varepsilon_{1}^{g}&= (-1)^{g}d^{\top}\widehat{\Sigma}_{w,\mathbf{c}}^{-1}(\mu_{g}-\overline{X}_{g})- 
d^{\top}\widehat{\Sigma}_{w,\mathbf{c}}^{-1}d + \delta^\top \Sigma_{w}^{-1}\delta\\
\varepsilon_{2} &= d^{\top}\widehat{\Sigma}_{w,\mathbf{c}}^{-1}\, \Sigma_{w}\,\widehat{\Sigma}_{w,\mathbf{c}}^{-1}d-\delta^\top \Sigma^{-1}_{w}\delta.
\end{align*}
Then the error terms $\varepsilon_1$ and $\varepsilon_2$ have the following upper bounds with probability at least $1-\eta:$
\begin{align*}
    &
  |\varepsilon_{1}^{g}|\leq 
C\,K_{s}^{2}(\|\Sigma_{w}^{-1/2}\delta\|_{2}
 +\|\Sigma_{w}^{-1/2}\delta\|_{2}^{2})\sqrt{\frac{\log(\eta^{-1})+p}{m}},
\end{align*}
and 
\begin{align*}
    |\varepsilon_{2}|&\leq  C\,K_{s}^{2}(\|\Sigma_{w}^{-1/2}\delta\|_{2}
 +\|\Sigma_{w}^{-1/2}\delta\|_{2}^{2})\sqrt{\frac{\log(\eta^{-1})+p}{m}}.
\end{align*}
Here, $C>0$ is an absolute constant, and $K_s = \{s\log(1+s^{-1})\}^{-1/2}$ is the sub-Gaussian norm of $Q_{i,j}^{g}/\sqrt{s}$- the entries of the compression matrices. 
\end{theorem}

\begin{proof}[Proof of Theorem~\ref{thm:ErrorDecomp}]
We have
\begin{align*}
    |\varepsilon_{1}^{g}|&\leq \underbrace{|d^{\top}\widehat{\Sigma}_{w,\mathbf{c}}^{-1}(\mu_{g}-\overline{X}_{g})|}_{(I)}+
     \underbrace{|d^{\top}\widehat{\Sigma}_{w,\mathbf{c}}^{-1}d-\delta^\top \Sigma^{-1}_{w}\delta|}_{(II)}\\
\end{align*}

We first bound $(I)$. Consider
\begin{align*}
    |(I)|=|d^{\top}\widehat{\Sigma}_{w,\mathbf{c}}^{-1}(\mu_{g}-\overline{X}_{g})|&= 
    |d^{\top}\Sigma_{w}^{-1/2}(\Sigma_{w}^{1/2}\widehat{\Sigma}_{w,\mathbf{c}}^{-1}\Sigma_{w}^{1/2})\Sigma_{w}^{-1/2}(\mu_{g}-\overline{X}_{g})|\\
    &\leq \underbrace{\|d^{\top}\Sigma_{w}^{-1/2}\|_{2}}_{A_1}\,\underbrace{\|\Sigma_{w}^{1/2}\widehat{\Sigma}_{w,\mathbf{c}}^{-1}\Sigma_{w}^{1/2}\|_{\text{op}}\,}_{A_2} \underbrace{\|\Sigma_{w}^{-1/2}(\mu_{g}-\overline{X}_{g})\|_{2}}_{A_3}.
\end{align*}
We bound $A_1 - A_3$ separately.  

For $A_1$, by Assumptions~\ref{assump:Normality} and \ref{assump:EqualClass}, 
$\Sigma_{w}^{-1/2}d \sim N(\Sigma_{w}^{-1/2}\delta, n^{-1}I_{p})$. By the triangle inequality and Proposition 1.1 of \citet{hsu2012tail}, the following holds with probability at least $1-\eta$ for any $\eta\in (0, e^{-1}):$ 
\begin{align*}
    \|\Sigma_{w}^{-1/2}d\|_{2} &\leq \|\Sigma_{w}^{-1/2}\delta\|_{2} +\|\Sigma_{w}^{-1/2}(d-\delta)\|_{2}\\
    &\leq \|\Sigma_{w}^{-1/2}\delta\|_{2}+
    \bigg(\frac{p}{n}+\frac{2\sqrt{p \log(\eta^{-1})}}{n}+\frac{2\log(\eta^{-1})}{n} \bigg)^{1/2}\\
&\leq \|\Sigma_{w}^{-1/2}\delta\|_{2} +\bigg(
\frac{p\log(\eta^{-1})}{n}+\frac{2\sqrt{p \log(\eta^{-1})}}{n}+\frac{2p\log(\eta^{-1})}{n}\bigg)^{1/2}\\
&\leq \|\Sigma_{w}^{-1/2}\delta\|_{2}+C\sqrt{\frac{p\log(\eta^{-1})}{n}}
\end{align*}

We now bound $A_2$. By Theorem 5, the following inequality holds with probability at least $1-\eta/3:$
$$
\|\Sigma_{w}^{1/2}\widehat{\Sigma}_{w,\mathbf{c}}^{-1}\Sigma_{w}^{1/2}\|_{\text{op}}
\leq
\|I_{p}\|_{\text{op}}+\|\Sigma_{w}^{1/2}\widehat{\Sigma}_{w,\mathbf{c}}^{-1}\Sigma_{w}^{1/2}- I_{p}\|_{\text{op}}
\leq
1+ C_2 \, K_{s}^{2}\, \sqrt{\frac{\log(\eta^{-1})+p}{m}}. 
$$

We now bound $A_3$. By Assumptions~\ref{assump:Normality} and \ref{assump:EqualClass}, $\Sigma_{w}^{-1/2}(\mu_g - \overline{X}_g)\sim N(0, n_{g}^{-1}I_{p})$. By Proposition 1.1 of \citet{hsu2012tail}, the following holds with probability at least $1-\eta:$
$$
\|\Sigma_{w}^{-1/2}(\mu_{g}-\overline{X}_{g})\|_{2}\leq C\sqrt{\frac{p\log(\eta^{-1})}{n}}.
$$
Combining the bounds for $A_1$-$A_3$, with probability at least $1-\eta:$
\begin{equation}\label{eq:Epsilon1Part2}
\begin{split}
     |d^{\top}\widehat{\Sigma}_{w,\mathbf{c}}^{-1}(\mu_{g}-\overline{X}_{g})|&\leq 
     C\bigg( \|\Sigma_{w}^{-1/2}\delta\|_{2}+C\sqrt{\frac{p\log(\eta^{-1})}{n}}\bigg)
     \bigg(1+ C_2 \, K_{s}^{2}\,\sqrt{\frac{\log(\eta^{-1})+p}{m}}\bigg)
     \sqrt{\frac{ p \log(\eta^{-1})}{n}}\\
     &\leq C K_{s}^{2}\|\Sigma_{w}^{-1/2} \delta\|_{2} \sqrt{ \frac{p\log(\eta^{-1})}{n}},
\end{split}
\end{equation}
where the last inequality came from absorbing lower-order terms into the absolute constant $C$. 

We now bound $(II)$. By the triangle inequality and Theorems~\ref{thm:CovAndMean}--\ref{thm:MeanInvCovariance}, with probability at least $1-\eta$:
\begin{align*}
 |d^{\top}\widehat{\Sigma}_{w,\mathbf{c}}^{-1}d-\delta^\top \Sigma^{-1}_{w}\delta|&\leq 
 | d^{\top}\widehat{\Sigma}_{w,\mathbf{c}}^{-1}d - d^\top \Sigma_{w}^{-1}d|+|d^\top \Sigma_{w}^{-1}d - \delta^\top \Sigma_{w}^{-1}\delta|\\
 &\leq 
C_{1}\,K_{s}^{2}\,\|\Sigma_{w}^{-1/2}\delta\|_{2}^{2}\,  \sqrt{\frac{\log(\eta^{-1})+p}{m}}+
C_{2} \|\Sigma_{w}^{-1/2}\delta\|_{2}\sqrt{\frac{p \log(\eta^{-1})}{n}}\\
&\leq C\, (K_{s}^{2}\,\|\Sigma_{w}^{-1/2}\delta\|_{2}^{2}+\|\Sigma_{w}^{-1/2}\delta\|_{2})\sqrt{\frac{\log(\eta^{-1})+p}{m}}.
\end{align*}
For $s\leq 0.8$, we have $K_{s}\geq 1$. Thus,
\begin{equation}\label{eq:Epsilon1Part1}
\begin{split}
 |d^{\top}\widehat{\Sigma}_{w,\mathbf{c}}^{-1}d-\delta^\top \Sigma^{-1}_{w}\delta| &\leq
 C\, (K_{s}^{2}\,\|\Sigma_{w}^{-1/2}\delta\|_{2}^{2}+\|\Sigma_{w}^{-1/2}\delta\|_{2})\sqrt{\frac{\log(\eta^{-1})+p}{m}}\\
 &\leq 
  C\,K_{s}^{2}\,(\|\Sigma_{w}^{-1/2}\delta\|_{2}^{2}+\|\Sigma_{w}^{-1/2}\delta\|_{2})\sqrt{\frac{\log(\eta^{-1})+p}{m}}.
 \end{split}
\end{equation}

Combining \eqref{eq:Epsilon1Part2} and \eqref{eq:Epsilon1Part1} gives with probability at least $1-\eta:$
\begin{align*}
|\varepsilon_{1}^{g}|&\leq 
C_1\, K_{s}^{2}\,(\|\Sigma_{w}^{-1/2}\delta\|_{2}^{2}+\|\Sigma_{w}^{-1/2}\delta\|_{2})\sqrt{\frac{\log(\eta^{-1})+p}{m}}
+ C_2 K_{s}^{2}\|\Sigma_{w}^{-1/2} \delta\|_{2} \sqrt{ \frac{p\log(\eta^{-1})}{n}}\\
&\leq C\, K_{s}^{2}\,(\|\Sigma_{w}^{-1/2}\delta\|_{2}^{2}+\|\Sigma_{w}^{-1/2}\delta\|_{2})\sqrt{\frac{\log(\eta^{-1})+p}{m}},
\end{align*}
where the lower-order term has been absorbed into the absolute constant $C_1$.

We now focus on bounding $\varepsilon_2$. The triangle inequality gives
\begin{align*}
    |\varepsilon_{2}| = |d^{\top}\widehat{\Sigma}_{w,\mathbf{c}}^{-1}\, \Sigma_{w}\,\widehat{\Sigma}_{w,\mathbf{c}}^{-1}d-\delta^\top \Sigma_{w}^{-1}\delta|\leq \underbrace{|d^{\top}\widehat{\Sigma}_{w,\mathbf{c}}^{-1}\, \Sigma_{w}\,\widehat{\Sigma}_{w,\mathbf{c}}^{-1}d
    -
    d^{\top} \Sigma_{w}^{-1}d|}_{A_1}
    +
    \underbrace{|d^{\top} \Sigma_{w}^{-1}d- \delta^\top \Sigma_{w}^{-1}\delta|}_{A_2}.
\end{align*}
We bound $A_1$-$A_2$ separately. 

First consider $A_1$. Using identity $I_{p}= \Sigma_{w}^{-1/2}\Sigma_{w}^{1/2}$ gives
\begin{align*}
    |A_1|=|d^\top \widehat{\Sigma}_{w, \mathbf{c}}^{-1}\Sigma_{w}\widehat{\Sigma}_{w, \mathbf{c}}^{-1} d - d^\top \Sigma_{w}d|&
    =
    |d^\top \Sigma_{w}^{-1/2}(\Sigma_{w}^{1/2} \widehat{\Sigma}_{w,\mathbf{c}}^{-1}\Sigma_{w}^{1/2})
   (\Sigma_{w}^{1/2} \widehat{\Sigma}_{w,\mathbf{c}}^{-1}\Sigma_{w}^{1/2})
    \Sigma_{w}^{-1/2}d- d^\top \Sigma_{w}^{-1}d|\\
    &\leq \|\Sigma_{w}^{-1/2} d\|_{2}^{2}\,\|(\Sigma_{w}^{1/2}\widehat{\Sigma}_{w,\mathbf{c}}^{-1}\Sigma_{w}^{1/2})^{2}-I_{p}\|_{\text{op}}.
\end{align*}
Let $A= \Sigma_{w}^{1/2}\widehat{\Sigma}_{w,\mathbf{c}}^{-1}\Sigma_{w}^{1/2}$. Then $\|(\Sigma_{w}^{1/2}\widehat{\Sigma}_{w,\mathbf{c}}^{-1}\Sigma_{w}^{1/2})^{2}-I_{p}\|_{\text{op}}$ is bounded above by
\begin{align*}
    \|I_{p}-A^{2}\|_{\text{op}}&=\|(I_{p}+A)(I_{p}-A)\|_{\text{op}}\\
    &\leq\|2 I_{p}+(A-I_{p})\|_{\text{op}}\|I_{p}-A\|_{\text{op}}\\
    &\leq [2+ \|I_{p}-A\|_{\text{op}}]\,\|I_{p}-A\|_{\text{op}}.
\end{align*}
Using the assumption that $\|I_{p}-A\|_{\text{op}} < 1$ and Theorem~\ref{thm:InvCov}, we have with probability at least $1-\eta:$
\begin{equation}\label{eq:Epsilon2Part1}
 \|I_{p}-A^{2}\|_{\text{op}}< 3 \|I_{p}-A\|_{\text{op}}\leq C\,K_{s}^{2}\,\sqrt{\frac{\log(\eta^{-1})+p}{m}}
\end{equation}
for some absolute constant $C>0.$

By Theorem \ref{thm:MeanInvCovariance}, the following holds with probability at least $1-\eta/2:$
\begin{equation}\label{eq:Epsilon2Part2}
\|\Sigma_{w}^{-1/2} d\|_{2}^{2}\leq \|\Sigma_{w}^{-1/2}\delta\|_{2}^{2}+|d^\top \Sigma_{w}^{-1}d - \delta^\top \Sigma_{w}^{-1}\delta| \leq \|\Sigma_{w}^{-1/2}\delta\|_{2}^{2}+ C \|\Sigma_{w}^{-1/2}\delta\|_{2} \sqrt{\frac{p\,\log(\eta^{-1})}{n}}.
\end{equation}
Combining \eqref{eq:Epsilon2Part1} and \eqref{eq:Epsilon2Part2} proves that the following bound on $A_1$ holds with probability at least $1-\eta:$
\begin{equation}\label{eq:Epsilon2Bound}
\begin{split}
 |d^\top \widehat{\Sigma}_{w, \mathbf{c}}^{-1}\Sigma_{w}\widehat{\Sigma}_{w, \mathbf{c}}^{-1} d - d^\top \Sigma_{w}d| &\leq \bigg(
 \|\Sigma_{w}^{-1/2}\delta\|_{2}^{2}+ C \|\Sigma_{w}^{-1/2}\delta\|_{2} \sqrt{\frac{p\,\log(\eta^{-1})}{n}}
 \bigg)\,C\,K_{s}^{2}\sqrt{\frac{\log(\eta^{-1})+p}{m}}\\
 &\leq C\,K_{s}^{2}(\|\Sigma_{w}^{-1/2}\delta\|_{2}
 +\|\Sigma_{w}^{-1/2}\delta\|_{2}^{2})\sqrt{\frac{\log(\eta^{-1})+p}{m}}.
 \end{split}
\end{equation}

To bound $A_2$, Theorem~\ref{thm:MeanInvCovariance} proves that with probability at least $1-\eta/2,$
\begin{align*}
   |A_2|= |d^{\top} \Sigma_{w}^{-1}d- \delta^\top \Sigma_{w}^{-1}\delta|\leq C \|\Sigma_{w}^{-1}\delta\|_{2}\sqrt{\frac{p \log(\eta^{-1})}{n}}.
\end{align*}

Since $A_2$ is a smaller-order term compared to \eqref{eq:Epsilon2Bound}, we absorb it into the absolute constant $C$. Thus, with probability at least $1-\eta:$
$$
|\varepsilon_{2}|\leq C\,K_{s}^{2}(\|\Sigma_{w}^{-1/2}\delta\|_{2}
 +\|\Sigma_{w}^{-1/2}\delta\|_{2}^{2})\sqrt{\frac{\log(\eta^{-1})+p}{m}}.
$$
This completes the proof.\qedhere
\end{proof}

\vspace{.2in}

\begin{theorem}\label{thm:CovAndMean}
Let the samples $X\in \mathbb{R}^{n\times p}$ be distributed according to Assumption~\ref{assump:Normality}.
Let $d$ and $\delta \in \mathbb{R}^{p}$ be as in Definition~\ref{def:MeanAndCov}, and let $\widehat{\Sigma}_{w,\mathbf{c}}$ be the compressed within-group covariance matrix. Then with probability at least $1-\eta,$
\begin{align*}
  | d^{\top}\widehat{\Sigma}_{w, \mathbf{c}}^{-1}d
-
d^{\top}\Sigma_{w}^{-1}d|\leq
C\,K_{s}^{2}\,\|\Sigma_{w}^{-1/2}\delta\|_{2}^{2}\, \sqrt{\frac{\log(\eta^{-1})+p}{m}} ,
\end{align*}
where $C>0$ is an absolute constant, and $K_{s}=\{s\log(1+s^{-1})\}^{-1/2}$ is the sub-Gaussian norm of $Q_{i,j}^{g}/\sqrt{s}$.
\end{theorem}

\begin{proof}[Proof of Theorem~\ref{thm:CovAndMean}]
 We have
\begin{align*}
    | d^{\top}\widehat{\Sigma}_{w, \mathbf{c}}^{-1}d
-
d^{\top}\Sigma_{w}^{-1}d|
& =
 | d^{\top}\Sigma_{w}^{-1/2}\Sigma_{w}(\widehat{\Sigma}_{w, \mathbf{c}}^{-1}
-
\Sigma_{w}^{-1})\Sigma_{w}^{1/2}\Sigma_{w}^{-1/2}d|\\
&\leq \|\Sigma_{w}^{-1/2}d\|_{2}^{2}\, \|\Sigma_{w}^{1/2}\widehat{\Sigma}_{w, \mathbf{c}}^{-1}\Sigma_{w}^{1/2}-I_{p}\|_{\text{op}}.
\end{align*}
By Theorem~\ref{thm:InvCov}, with probability at least $1-\eta/2,$
$$
\|\Sigma_{w}^{1/2}\widehat{\Sigma}_{w, \mathbf{c}}^{-1}\Sigma_{w}^{1/2}-I_{p}\|_{\text{op}}\leq C\,K_{s}^{2}\,\sqrt{\frac{\log(\eta^{-1})+p}{m}}.
$$
for some absolute constant $C>0$, and where $K_{s}=\{s\log(1+s^{-1})\}^{-1/2}$ is the sub-Gaussian norm of $Q_{i,j}^{g}/\sqrt{s}$ by Lemma~\ref{lem:sub-GaussianNorm}.

By the triangle inequality and Theorem~\ref{thm:MeanInvCovariance}, with probability at least $1-\eta:$
\begin{align*}
\|\Sigma_{w}^{-1/2} d\|_{2}^{2} &=
|d^\top \Sigma_{w}^{-1}d - \delta^\top \Sigma_{w}^{-1}\delta + \delta^\top \Sigma_{w}^{-1}\delta|\\
&\leq\|\Sigma_{w}^{-1/2}\delta\|_{2}^{2} + |d^\top \Sigma_{w}^{-1}d - \delta^\top \Sigma_{w}^{-1}\delta| 
\leq \|\Sigma_{w}^{-1/2}\delta\|_{2}^{2} + C \|\Sigma_{w}^{-1/2}\delta\|_{2}\,\sqrt{\frac{p \log(\eta^{-1})}{n}}.
\end{align*}

Combining the two displays above and absorbing the lower order term into the absolute constant $C$, we have that with probability at least $1-\eta$
\begin{align*}
 | d^{\top}\widehat{\Sigma}_{w, \mathbf{c}}^{-1}d
-
d^{\top}\Sigma_{w}^{-1}d|&\leq \bigg(\|\Sigma_{w}^{-1/2}\delta\|_{2}^{2} + C \|\Sigma_{w}^{-1/2}\delta\|_{2}\,\sqrt{\frac{p \log(\eta^{-1})}{n}}\bigg)C_{1}\,K_{s}^{2}\,\sqrt{\frac{\log(\eta^{-1})+p}{m}}\\
&\leq C\,K_{s}^{2}\|\Sigma_{w}^{-1/2}\delta\|_{2}^{2}\,\sqrt{\frac{\log(\eta^{-1})+p}{m}}.
\end{align*}
\qedhere
\end{proof}

\begin{theorem}\label{thm:MeanInvCovariance}
Let the samples in $X\in \mathbb{R}^{n\times p}$ be distributed according to Assumption~\ref{assump:Normality}, and let $d$ and $\delta$ be as in Definition~\ref{def:MeanAndCov}. Then for $\eta \in (0, e^{-1})$, the following upper bound holds with probability at least $1-\eta,$
\begin{align*}
    |d^{\top} \Sigma_{w}^{-1}d - \delta^\top \Sigma_{w}^{-1}\delta|
  \leq  
C \|\Sigma_{w}^{-1/2}\delta\|_{2}\sqrt{\frac{p \log(\eta^{-1})}{n}}
\end{align*}
for some absolute constant $C>0$.
\end{theorem}

\begin{proof}[Proof of Theorem~\ref{thm:MeanInvCovariance}]
 Completing the square gives
\begin{align*}
    |d^{\top} \Sigma_{w}^{-1}d - \delta^\top \Sigma_{w}^{-1}\delta|&=
  |(d-\delta)^{\top} \Sigma_{w}^{-1}(d-\delta)+2 (d-\delta)^{\top}\Sigma_{w}^{-1}\delta|
  &\leq \| \Sigma_{w}^{-1/2}(d-\delta)\|_{2}^{2}+2 \|\Sigma_{w}^{-1/2}\delta\|_{2} \,\|\Sigma_{w}^{-1/2}(d-\delta)\|_{2}.
\end{align*}
Assumptions~\ref{assump:Normality} and \ref{assump:EqualClass} give $\Sigma_{w}^{-1/2}(d-\delta)\sim N(0,n^{-1}I_{p})$. By Proposition 1.1 of \citet{hsu2012tail}, with probability at least $1-\eta,$
$$
\|\Sigma_{w}^{-1/2}(d-\delta)\|_2^2 \leq \frac{p}{n}+\frac{2\sqrt{p \log(\eta^{-1})}}{n}+\frac{2\log(\eta^{-1})}{n}.
$$
For $\eta \in (0, e^{-1})$, we have $\log(\eta^{-1}) \geq 1$. It follows that 
\begin{align*}
\|\Sigma_{w}^{-1/2}(d-\delta)\|_2^2 &\leq
\frac{p}{n}+\frac{2\sqrt{p \log(\eta^{-1})}}{n}+\frac{2\log(\eta^{-1})}{n}\\
&\leq
\frac{p\log(\eta^{-1})}{n}+\frac{2\sqrt{p \log(\eta^{-1})}}{n}+\frac{2p\log(\eta^{-1})}{n}\\
&\leq
C\frac{ p \log(\eta^{-1})}{n}.
\end{align*}
 Then
\begin{align*} 
\| \Sigma_{w}^{-1/2}(d-\delta)\|_{2}^{2}+2 \|\Sigma_{w}^{-1/2}\delta\|_{2} \,\|\Sigma_{w}^{-1/2}(d-\delta)\|_{2}& \leq C_1\frac{ p \log(\eta^{-1})}{n} + C_2 \|\Sigma_{w}^{-1/2}\delta\|_{2}\sqrt{\frac{p \log(\eta^{-1})}{n}}\\
&\leq C \|\Sigma_{w}^{-1/2}\delta\|_{2}\sqrt{\frac{p \log(\eta^{-1})}{n}}.\qedhere
\end{align*}
\end{proof}

\vspace{.2in}

\begin{theorem}[Inverse Covariance Bound]\label{thm:InvCov}
Let the samples $X\in \mathbb{R}^{n\times p}$ be distributed according to Assumption~\ref{assump:Normality} with shared covariance $\Sigma_{w}\in \mathbb{R}^{p\times p}$.
Let $\widehat{\Sigma}_{w,\mathbf{c}}$ be the within-group sample covariance matrix of the compressed data with sparsity parameter $s>0.$ Then with probability at least $1-\eta$,
$$
\| I_{p}- \Sigma_{w}^{1/2}\widehat{\Sigma}_{w, \mathbf{c}}^{-1}\Sigma_{w}^{1/2}\|_{\text{op}}\leq C\,K_{s}^{2}\sqrt{\frac{\log(\eta^{-1})+p}{m}}
$$
for some absolute constant $C>0$, and where $K_s = \{s\log(1+s^{-1})\}^{-1/2}$ is the sub-Gaussian norm of $Q_{i,j}^{g}/\sqrt{s}$.
\end{theorem}

\begin{proof}
For $A:= \Sigma_{w}^{-1/2} \widehat{\Sigma}_{w, \mathbf{c}} \Sigma_{w}^{-1/2}$, the above is of the form $\|\Sigma_w^{-1}\|_{\text{op}}\|A^{-1}-I\|_{\text{op}}$.
By Theorem~\ref{thm:SigmaBound}, $\|I-A\|_{\text{op}}<1$ with high probability. Then $A$ has the geometric sum expansion of its inverse
$
A^{-1} = \sum_{k=0}^{\infty}(I-A)^{k}.
$
Thus,
\begin{align*}
\|I_{p}-A^{-1}\|_{\text{op}}&=\bigg\| I_{p} -  \sum_{k=0}^{\infty}(I_{p}-A)^{k}\bigg\|_{\text{op}}=\bigg\| \sum_{k=1}^{\infty} (I_{p}-A)^{k}\bigg\|_{\text{op}}\\
&\leq \sum_{k=1}^{\infty} \|I_{p}-A\|^{k}_{\text{op}}
= \sum_{k=0}^{\infty} \|I_{p}-A\|^{k}_{\text{op}} -1\\
&= \frac{1}{1-\|I_{p}-A\|_{\text{op}}}-1= \frac{\|I_{p}-A\|_{\text{op}}}{1-\|I_{p}-A\|_{\text{op}}}=\|I_p - A\|_{\text{op}}+o_{p}(\|I_p - A\|_{\text{op}}),
\end{align*}
where the last equality comes from the Taylor Expansion of the function $t/(1-t)$ centered at $0$.

Applying Theorem~\ref{thm:SigmaBound} and absorbing the lower-order $o_{p}(\|I_{p}-A\|_{\text{op}})$ into the absolute constant $C$ proves that with probability at least $1-\eta,$
$$
\| I_{p}- \Sigma_{w}^{1/2}\widehat{\Sigma}_{w, \mathbf{c}}^{-1}\Sigma_{w}^{1/2}\|_{\text{op}}\leq C\,K_{s}^{2}\sqrt{\frac{\log(\eta^{-1})+p}{m}}
$$
\qedhere
\end{proof}

\vspace{.2in}

\begin{theorem}[Covariance Bound]\label{thm:SigmaBound}
Let the samples $X\in \mathbb{R}^{n\times p}$ be distributed according to Assumption~\ref{assump:Normality} with shared covariance $\Sigma_{w}\in \mathbb{R}^{p\times p}$.
Let $\widehat{\Sigma}_{w,\mathbf{c}}\in \mathbb{R}^{p \times p}$ be the within-group sample covariance matrix of the compressed data with sparsity parameter $s>0.$ Then with probability at least $1-\eta$:
\begin{equation}
\begin{split}
 \|\Sigma_{w}^{-1/2}\widehat{\Sigma}_{w, \mathbf{c}}\Sigma_{w}^{-1/2}- I_{p}\|_{\text{op}}\leq 
C \, K_{s}^{2}\, \sqrt{\frac{\log(\eta^{-1})+p}{m}},
\end{split}
\end{equation}
for some absolute constant $C>0$, and where $K_{s}=\{s\log(1+s^{-1})\}^{-1/2}$ is the sub-Gaussian norm of $Q_{i,j}^{g}/\sqrt{s}$.
\end{theorem}

\begin{proof}[Proof of Theorem~\ref{thm:SigmaBound}]
By the definition of $\widehat{\Sigma}_{w, \mathbf{c}}$,
\begin{align*}
    \Sigma_{w}^{-1/2}\widehat{\Sigma}_{w, \mathbf{c}}\Sigma_{w}^{-1/2}&=\frac{1}{m}\sum_{g=1}^{2}\sum_{j=1}^{m_g}\Sigma_{w}^{-1/2}(\mathbf{x}_{j,\mathbf{c}}^{g}-\overline{X}_{g})(\mathbf{x}_{j,\mathbf{c}}^{g}-\overline{X}_{g})^{\top}\Sigma_{w}^{-1/2}\\
    &=
    \frac{1}{m}\sum_{g=1}^{2}\sum_{j=1}^{m_g}\Sigma_{w}^{-1/2}\bigg(\frac{1}{\sqrt{n_g s}} \sum_{i=1}^{n_g}Q^{g}_{j,i}(\mathbf{x}_{i}^{g}-\overline{X}_{g})+\overline{X}_{g}-\overline{X}_{g}\bigg)\\
    &\qquad\qquad\qquad\bigg(\frac{1}{\sqrt{n_g s}} \sum_{\ell=1}^{n_g}Q^{g}_{j,\ell}(\mathbf{x}_{\ell}^{g}-\overline{X}_{g})+\overline{X}_{g}-\overline{X}_{g}\bigg)^{\top}\Sigma_{w}^{-1/2}\\
    &=    
    \frac{1}{m}\sum_{g=1}^{2}\sum_{j=1}^{m_g}\bigg(\frac{1}{\sqrt{n_g s}} \sum_{i=1}^{n_g}Q^{g}_{j,i}\Sigma_{w}^{-1/2}(\mathbf{x}_{i}^{g}-\mu_{g}+\mu_{g}-\overline{X}_{g})
    \bigg)\\
    &\qquad\qquad\qquad\bigg(\frac{1}{\sqrt{n_g s}} \sum_{\ell=1}^{n_g}Q^{g}_{j,\ell}\Sigma_{w}^{-1/2}(\mathbf{x}_{\ell}^{g}-\mu_{g}+\mu_{g}-\overline{X}_{g})
    \bigg)^{\top}\\
    &= 
      \underbrace{\frac{1}{m}\sum_{g=1}^{2}\sum_{j=1}^{m_g}\bigg(\frac{1}{\sqrt{n_g s}} \sum_{i=1}^{n_g}Q^{g}_{j,i}\Sigma_{w}^{-1/2}(\mathbf{x}_{i}^{g}-\mu_{g})
    \bigg)\bigg(\frac{1}{\sqrt{n_g s}} \sum_{\ell=1}^{n_g}Q_{j,\ell}^{g}\Sigma_{w}^{-1/2}(\mathbf{x}_{\ell}^{g}-\mu_{g}) \bigg)^{\top}}_{A_1}\\
    &\qquad\qquad -\underbrace{\frac{1}{m}\sum_{g=1}^{2}\sum_{j=1}^{m_g}\bigg(\frac{1}{\sqrt{n_g s}} \sum_{i=1}^{n_g}Q^{g}_{j,i}\Sigma_{w}^{-1/2}(\mathbf{x}_{i}^{g}-\mu_{g})
    \bigg) \bigg( \frac{1}{\sqrt{n_g s}}\sum_{\ell=1}^{n_g} Q_{ j, \ell}^{g}\Sigma_{w}^{-1/2}(\mu_{g}-\overline{X}_{g})\bigg)^\top}_{A_{2}}\\
    &\qquad\qquad
     -\underbrace{\frac{1}{m}\sum_{g=1}^{2}\sum_{j=1}^{m_g}\bigg(\frac{1}{\sqrt{n_g s}} \sum_{i=1}^{n_g}Q_{j,i}^{g}\Sigma_{w}^{-1/2}(\mu_{g}-\overline{X}_{g})
    \bigg)\bigg(\frac{1}{\sqrt{n_g s}} \sum_{\ell=1}^{n_g}Q^{g}_{j,\ell}\Sigma_{w}^{-1/2}(\mathbf{x}_{\ell}^{g}-\mu_{g}) \bigg)^{\top}}_{A_{3}}\\
    &\qquad\qquad +\underbrace{\frac{1}{m}\sum_{g=1}^{2}\sum_{j=1}^{m_g} \bigg(\frac{1}{\sqrt{n_g s}} \sum_{i=1}^{n_g}Q^{g}_{j,i}\Sigma_{w}^{-1/2}(\mu_{g}-\overline{X}_{g})
    \bigg)\bigg(\frac{1}{\sqrt{n_g s}} \sum_{\ell=1}^{n_g}Q^{g}_{j,\ell}\Sigma_{w}^{-1/2}(\mu_{g}-\overline{X}_{g})
    \bigg)^{\top}}_{A_{4}}.
 \end{align*}
We bound $A_1-A_4$ separately. We do this by considering a fixed $v\in\mathbb{R}^{p}$ with norm $\|v\|_{2}=1.$ We first bound each $v^{\top}A_iv$ and then generalize to a norm bound using an $\epsilon$-net argument.

Consider
\begin{align*}
v^{\top}A_1v&=\frac{1}{2}\sum_{g=1}^{2}v^\top \bigg[ \frac{1}{\sqrt{n_g}}\frac{1}{\sqrt{n_g}}\sum_{i, \ell=1}^{n_g}\bigg\{\frac{1}{m_g}\sum_{j=1}^{m_g} \frac{1}{s} Q^{g}_{j,i}Q^{g}_{j,\ell}\bigg\}\Sigma_{w}^{-1/2}(\mathbf{x}_{i}^{g}-\mu_{g})(\mathbf{x}_{j}^{g}-\mu_{g})^{\top}\Sigma_{w}^{-1/2}\bigg]v\\
&=
 \frac{1}{2}\sum_{g=1}^{2}\sum_{i, \ell=1}^{n_g}\bigg\{\frac{1}{m_g}\sum_{j=1}^{m_g} \frac{1}{s} Q^{g}_{j,i}Q^{g}_{j,\ell}\bigg\}\frac{1}{\sqrt{n_g}} \left<\Sigma_{w}^{-1/2}(\mathbf{x}_{i}^{g}-\mu_{g})\,,\, v\right>\,
\frac{1}{\sqrt{n_g}}\left<\Sigma_{w}^{-1/2}(\mathbf{x}_{j}^{g}-\mu_{g})\,,\, v\right>\\
&=\frac{1}{2}\sum_{g=1}^{2}\frac{1}{n_g\, m_g}{Z}^{g}{}^{\top} R_{g} Z^{g},
\end{align*}
 where $Z^{g}\in \mathbb{R}^{n_{g}}$ is the vector with $i$-th coordinate $\left<\Sigma_{w}^{-1/2}(\mathbf{x}_{i}^{g}-\mu_{g}), v\right>$, and $R_{g}= \frac{1}{s}Q^{g}{}^{\top}Q^{g}\in \mathbb{R}^{n_g \times n_g}$. By Assumption~\ref{assump:Normality}, $Z^{g}\sim N(0, I_{n_g})$. By Lemma~\ref{lem:UnconditionalHW}, with probability at least $1-\eta:$
\begin{equation}\label{eq:DominantTerm}
|v^\top(A_1-I_{p})v| = |v^{\top}A_1v - 1|
=\bigg| \frac{1}{2}\sum_{g=1}^{2}\frac{1}{n_g\, m_g} Z^{g}{}^{\top} R_{g} Z^{g} - 1\bigg|\leq C\, K_{s}^{2} \sqrt{\frac{\log(\eta^{-1})}{m}}.
\end{equation}

The terms $A_{2}$ and $A_{3}$ are transposes of each other, and so we handle them simultaneously. Left and right multiplying by $v$ gives
\begin{align*}
&\frac{1}{2}\sum_{g=1}^{2}\,v^{\top}\bigg[\frac{1}{m_g}\sum_{j=1}^{m_g}\bigg(\frac{1}{\sqrt{n_g s}} \sum_{i=1}^{n_g}Q^{g}_{i,j}\Sigma_{w}^{-1/2}(\mathbf{x}_{i}^{g}-\mu_{g})
    \bigg) \bigg( \frac{1}{\sqrt{n_g s}}\sum_{\ell=1}^{n_g} Q^{g}_{\ell, j}\Sigma_{w}^{-1/2}(\mu_{g}-\overline{X}_{g})\bigg)^\top\bigg]v\\
&=
\frac{1}{m}\sum_{g=1}^{2}\,\sum_{j=1}^{m_g}\bigg( \frac{1}{\sqrt{n_g s}} \sum_{i=1}^{n_g}Q^{g}_{i,j}\left<\Sigma_{w}^{-1/2}(\mathbf{x}_{i}^{g}-\mu_g)\,,\, v\right>\bigg)
\bigg(\frac{1}{\sqrt{n_g s}} \sum_{\ell=1}^{n_g} Q^{g}_{\ell, j}\bigg)\left<\Sigma_{w}^{-1/2}(\mu_g-\overline{X}_{g})\,,\, v\right>.\\
\end{align*}

By Assumption~\ref{assump:Normality}, $\left<\Sigma_{w}^{-1/2}(\mu_g-\overline{X}_{g})\,,\, v\right> \sim N(0, n_{g}^{-1})$. By the Gaussian concentration inequality, with probability at least $1-\eta/3$:
\begin{equation}\label{eq:MiddleTerm1}
\bigg|\left<\Sigma_{w}^{-1/2}(\mu_g-\overline{X}_{g})\,,\, v\right>\bigg|\leq C\sqrt{\frac{\log(\eta^{-1})}{n_g}} = C'\sqrt{\frac{\log(\eta^{-1})}{n}}
\end{equation}
for some absolute constants $C, C'>0$. The last equality comes from Assumption~\ref{assump:EqualClass}.

By the general Hoeffding's Inequality, Theorem 2.6.3 of \citet{vershynin2018high}, with probability at least $1-\eta/3:$
\begin{equation}\label{eq:MiddleTerm2}
\bigg|\frac{1}{\sqrt{n_g s}}\sum_{\ell=1}^{n_g}Q^{g}_{\ell, j}\bigg|\leq C K_{s} \sqrt{\log(\eta^{-1})},
\end{equation}
where $K_{s}= \{s\log(1+s^{-1})\}^{-1/2}$ is the sub-Gaussian norm of $Q^{g}_{i,j}/\sqrt{s}$ by Lemma~\ref{lem:sub-GaussianNorm}.

Lastly, 
\begin{equation}\label{eq:ZandQterms}
\frac{1}{m}\sum_{g=1}^{2}\,\sum_{j=1}^{m_g}\bigg( \frac{1}{\sqrt{n_g s}} \sum_{i=1}^{n_g}Q^{g}_{i,j}\left<\Sigma_{w}^{-1/2}(\mathbf{x}_{i}^{g}-\mu_g)\,,\, v\right>\bigg)=\frac{1}{2} \sum_{g=1}^{2}\frac{1}{\sqrt{n_g}}\sum_{i=1}^{n_g}\bigg(\frac{1}{m_g}\sum_{j=1}^{m_g}\frac{1}{\sqrt{s}}Q_{i,j}^{g}\bigg) Z_{i}^{g},
\end{equation}
where the $Z_{i}^{g}$ are as above. Let $X_{ig} =m_g^{-1}\sum_{j=1}^{m_g}Q_{i,j}^{g}/\sqrt{s}$, then by Lemma~\ref{lem:sub-GaussianNorm} the sub-Gaussian norm of $X_{ig}$ is $K_s/\sqrt{m}$. Conditioning on vectors $Z^{g}=(Z_1^{g},\dots, z_{n_g}^g)$, and applying Hoeffding's Inequality to $Q_{i,j}^{g}$ gives that with probability at least $1-\eta/6:$
$$
\bigg|\frac{1}{2} \sum_{g=1}^{2}\frac{1}{\sqrt{n_g}}\sum_{i=1}^{n_g}\bigg(\frac{1}{m_g}\sum_{j=1}^{m_g}\frac{1}{\sqrt{s}}Q_{i,j}^{g}\bigg) Z_{i}^{g}\bigg|=
\bigg|\sum_{g=1}^{2}\sum_{i=1}^{n_g}\frac{1}{2\sqrt{n_g}} Z_{i}^{g}X_{ig}\bigg|
\leq  CK_{s}\, \sqrt{\frac{\log(\eta^{-1})}{m}} \bigg(\frac{\|Z^{1}\|^{2}_2+\|Z^{2}\|^{2}_2}{n}\bigg)^{1/2}.
$$
 Let $Z= \begin{pmatrix} {Z^{1}}^{\top}&{Z^{2}}^{\top}\end{pmatrix}^\top \in \mathbb{R}^{n}$. By Theorem 3.1.1 of \citet{vershynin2018high},
$$
\mathbb{P}\bigg(\bigg|\frac{1}{\sqrt{n}}\|Z\|_{2} - 1\bigg| \geq t \bigg)=
\mathbb{P}(|\|Z\|_{2} - \sqrt{n}| \geq \sqrt{n}\,t )\leq 2\exp ( -c\,n\,t^2).
$$
This is equivalent to the following upper bound holding with probability at least $1-\eta/6:$
\begin{align*}
\bigg(\frac{\|Z^{1}\|^{2}_2+\|Z^{2}\|^{2}_2}{n}\bigg)^{1/2}=\frac{1}{\sqrt{n}}\|Z\|_{2}&\leq 1 + C\sqrt{\frac{\log(\eta^{-1})}{n}},
\end{align*}
where $C>0$ is an absolute constant. Combining the above two displays gives the following bound for \eqref{eq:ZandQterms}, which holds with probability at least $1-\eta/3:$
\begin{equation}\label{eq:MiddleTerm3}
\begin{split}
\bigg|\frac{1}{2} \sum_{g=1}^{2}\frac{1}{\sqrt{n_g}}\sum_{i=1}^{n_g}\bigg(\frac{1}{m_g}\sum_{j=1}^{m_g}\frac{1}{\sqrt{s}}Q_{i,j}^{g}\bigg) Z_{i}^{g}\bigg| \leq C_1 K_{s}\, \sqrt{\frac{\log(\eta^{-1})}{m}}
\bigg( 1+ C_2\sqrt{\frac{\log(\eta^{-1})}{n}}\bigg)\leq C\, K_{s}\,\sqrt{\frac{\log(\eta^{-1})}{m}}.
\end{split}
\end{equation}

Putting \eqref{eq:MiddleTerm1}, \eqref{eq:MiddleTerm2} and \eqref{eq:MiddleTerm3} together shows that with probability at least $1-\eta,$
\begin{align*}
&|v^{\top}A_2 v|\\
&=\bigg|\frac{1}{m}\sum_{g=1}^{2}\,\sum_{j=1}^{m_g}\bigg( \frac{1}{\sqrt{n_g s}} \sum_{i=1}^{n_g}Q^{g}_{i,j}\left<\Sigma_{w}^{-1/2}(\mathbf{x}_{i}^{g}-\mu_g)\,,\, v\right>\bigg)
\bigg(\frac{1}{\sqrt{n_g s}} \sum_{\ell=1}^{n_g} Q^{g}_{\ell, j}\bigg)\left<\Sigma_{w}^{-1/2}(\mu_g-\overline{X}_{g})\,,\, v\right>\bigg|\\
&\leq 
\bigg( C_{1} K_{s} \sqrt{\frac{\log(\eta^{-1})}{m}}\bigg)
\,C_2 K_{s} \sqrt{\log(\eta^{-1})}\,\bigg( C_{3}\sqrt{\frac{\log(\eta^{-1})}{n}}\bigg)\\
&\leq C K_{s}^{2}\, \frac{\log(\eta^{-1})}{\sqrt{n}}\sqrt{\frac{\log(\eta^{-1})}{m}}\leq 
C K_{s}^{2} \sqrt{\frac{\log(\eta^{-1})}{m}}. 
\end{align*}
We have used Assumption~\ref{assump:Sizeofm} in the last inequality. 

For $A_{4}$, left and right multiplying by $v$ gives
\begin{align*}
v^{\top}A_4v &= v^{\top}\bigg[ \frac{1}{2}\sum_{g=1}^{2}\frac{1}{m_g}\sum_{j=1}^{m_g} \bigg(\frac{1}{\sqrt{n_g p}} \sum_{i=1}^{n_g}Q^{g}_{i,j}\Sigma_{w}^{-1/2}(\mu_{g}-\overline{X}_{g})
    \bigg)\bigg(\frac{1}{\sqrt{n_g p}} \sum_{\ell=1}^{n_g}Q^{g}_{\ell,j}\Sigma_{w}^{-1/2}(\mu_{g}-\overline{X}_{g})
    \bigg)^{\top}\bigg]v\\
    & = \frac{1}{2}\sum_{g=1}^{2}
    \bigg\{\frac{1}{m_g}\sum_{j=1}^{m_g} \bigg(\frac{1}{\sqrt{n_g p}} \sum_{i=1}^{n_g}Q^{g}_{i,j}\bigg)^{2}\bigg\} \left<\Sigma_{w}^{-1/2}(\mu_{g}-\overline{X}_{g})\,,\, v\right>^{2},
\end{align*}
where the last equality is true since $\Sigma_{w}^{-1/2}(\mu_{g}- \overline{X}_{g})$ is independent of $i$, $j$, and $\ell$.

By Assumption~\ref{assump:Normality}, $\left<\Sigma_{w}^{-1/2}(\mu_g-\overline{X}_{g})\,,\, v\right>\sim N(0, n_{g}^{-1})$. The Gaussian concentration inequality proves that with probability at least $1-\eta/2:$
$$
\bigg|\left<\Sigma_{w}^{-1/2}(\mu_g-\overline{X}_{g})\,,\, v\right>\bigg|^2\leq C\frac{\log(\eta^{-1})}{n_g}.
$$
The squared terms
$$
\bigg(\frac{1}{\sqrt{n_g p}} \sum_{i=1}^{n_g}Q^{g}_{i,j}\bigg)^{2}
$$
are sub-Exponential because they are the squares of sub-Gaussian random variables. By Lemma 2.7.6 of \citet{vershynin2018high}, the sub-Exponential norm satisfies
$$
\bigg\| \bigg(\frac{1}{\sqrt{n_g p}} \sum_{i=1}^{n_g}Q^{g}_{i,j}\bigg)^{2}\bigg\|_{\Psi_1}= \bigg\|
\frac{1}{\sqrt{n_g p}} \sum_{i=1}^{n_g}Q^{g}_{i,j}\bigg\|_{\Psi_2}^{2}= C\,K_{s}^{2},
$$
where $C>0$ is an absolute constant and  $K_{s}$ is the sub-Gaussian norm of $Q_{i,j}^{g}/\sqrt{s}$ by Lemma~\ref{lem:sub-GaussianNorm}. Thus, by Bernstein's Inequality, with probability at least $1-\eta$:
\begin{align*}
    \bigg| \frac{1}{m_g}\sum_{j=1}^{m_g} \bigg(\frac{1}{\sqrt{n_g s}} \sum_{i=1}^{n_g}Q^{g}_{j,i}\bigg)^{2}\bigg|\leq C\,K_{s}^{2}\,\max\bigg\{\frac{\log(\eta^{-1})}{m_g}\,,\,\sqrt{\frac{\log(\eta^{-1})}{m_g}}\bigg\} \leq C\,K_{s}^{2} \sqrt{\frac{\log(\eta^{-1})}{m_g}}.
\end{align*}
Combining the above displays, with probability at least $1-\eta$, $|v^\top A_{4} v|$ is bounded above by
$$
|v^{\top}A_{4}v|\leq C K_{s}^{2} \sqrt{\frac{\log(\eta)^{-1}}{m_g}}\,\frac{\log(\eta^{-1})}{n_{g}} \leq CK_s^2\, \frac{\log(\eta^{-1})}{n},
$$
where we have used Assumptions~\ref{assump:EqualClass} and \ref{assump:Sizeofm}.

Combining the above bounds for $A_1 -A_4$ shows that with probability at least $1-\eta:$
\begin{align*}
|v^{\top}( \Sigma_{w}^{-1/2}\widehat{\Sigma}_{w, \mathbf{c}}\Sigma_{w}^{-1/2} - I_{p})v|&\leq  C_1\, K_{s}^{2} \sqrt{\frac{\log(\eta^{-1})}{m}}+ 
 C_2\, K_{s}^{2} \sqrt{\frac{\log(\eta^{-1})}{m}}
+
C_3K_s^2\, \frac{\log(\eta^{-1})}{n}\\
&= C \, K_{s}^{2}\,\sqrt{\frac{\log(\eta^{-1})}{m}}.
\end{align*}

We now generalize to a norm bound via an $\varepsilon$-net argument. Let $\mathcal{N}$ be a $1/3$-net on the unit sphere of $\mathbb{R}^{p}.$ There exists a $1/3$-net such that $|\mathcal{N}|\leq 7^{p}$ (see Corollary 4.2.13 of \citet{vershynin2018high}). Thus, 
\begin{align*}
\mathbb{P}\bigg(\sup_{v\in \mathcal{N}}|v^\top( \Sigma_{w}^{-1/2}\widehat{\Sigma}_{w, \mathbf{c}}\Sigma_{w}^{-1/2}-I_{p})v|\geq t\bigg)
    &=\mathbb{P}\bigg( \bigcup_{v\in \mathcal{N}} \{|v^\top( \Sigma_{w}^{-1/2}\widehat{\Sigma}_{w, \mathbf{c}}\Sigma_{w}^{-1/2}-I_{p})v|\geq t\}\bigg)\\&\leq \sum_{v\in \mathcal{N}} \mathbb{P}(|v^\top( \Sigma_{w}^{-1/2}\widehat{\Sigma}_{w, \mathbf{c}}\Sigma_{w}^{-1/2}-I_{p})v|\geq t)\\
    &\leq \sum_{v\in \mathcal{N}}\exp\bigg( -\frac{C\, m\, t^2}{ K_{s}^{4}}\bigg)\\
    &=|\mathcal{N}| \exp\bigg(-\frac{C \,m\,t^2}{K_{s}^{4}}\bigg)\\
    &\leq \exp(p \log(7))\,\exp\bigg(-\frac{C \, m\,t^2}{K_{s}^{4}}\bigg)=
    \exp\bigg(C_1\,p- C_2\frac{m\,t^2}{K_{s}^{4}}\bigg).
\end{align*}

This tail inequality is equivalent to the following upper bound holding with probability at least $1-\eta:$
$$
\sup_{v\in \mathcal{N}}|v^\top( \Sigma_{w}^{-1/2}\widehat{\Sigma}_{w, \mathbf{c}}\Sigma_{w}^{-1/2}-I_{p})v|
\leq C_1\,K_{s}^{2}\sqrt{\frac{\log(\eta^{-1})+C_2 p}{m}}
\leq
C\,K_{s}^{2}\,\max\{1,\sqrt{C_2}\}\,\sqrt{\frac{\log(\eta^{-1})+p}{m}}.
$$
Absorbing $\max\{1, \sqrt{C_2}\}$ into the absolute constant $C_1$ gives a uniform bound on the $\varepsilon$-net $\mathcal{N}$. Applying Lemma~\ref{lem:NetQuadForm} proves the final reuslt.\qedhere
\end{proof}

\vspace{.2in}
\begin{lemma}[page 88 of \citet{vershynin2018high} ]\label{lem:NetQuadForm}
Let $\varepsilon\in [0,1/2)$. Then for any $\varepsilon$-net $\mathcal{N}$ of the unit sphere of $\mathbb{R}^{p}$, we have
$$
\sup_{v\in \mathcal{N}}|v^\top (\widehat{\Sigma}_{w, \mathbf{c}}-\Sigma_{w})v|\leq \|\widehat{\Sigma}_{w, \mathbf{c}}-\Sigma_{w}\|_{\text{op}}\leq \frac{1}{1-2\varepsilon}\, \sup_{v\in \mathcal{N}}|v^\top (\widehat{\Sigma}_{w, \mathbf{c}}-\Sigma_{w})v|.
$$
\end{lemma}


\begin{lemma}\label{lem:UnconditionalHW}
For $g=1,2$, let $Z^{g}\sim N(0, I_{n_g})$, let $Q^{g}\in \mathbb{R}^{m_g\times n_g}$ consist of i.i.d. sparse Rademacher random variables with sparsity parameter $s$, and let $R_{g}= {Q^{g}}^\top Q^{g}/s$. Then with probability at least $1-\eta$:
$$
   \bigg|\frac{1}{2}\sum_{g=1}^{2}\frac{1}{n_g\, m_g} {Z^{g}}^{\top} R_{g} Z^{g} - 1\bigg|
   \leq    C\,
 K_s^{2} \,\sqrt{\frac{\log(\eta^{-1})}{m}},
$$
where $C>0$ is an absolute constant, and $K_{s} = \{s \log(1+s^{-1})\}^{-1/2}$ is the sub-gaussian norm of $Q_{i,j}^{g}/\sqrt{s}$.
\end{lemma}

\begin{proof}[Proof of Lemma~\ref{lem:UnconditionalHW}]
By Lemma~\ref{lem:ConditionalHW}, with probability at least $1-\eta/2$:
\begin{equation}
 \begin{split}\label{eq:lemma2start}
   \bigg|\frac{1}{2}\sum_{g=1}^{2}\frac{1}{n_g\, m_g}{Z^{g}}^{\top} R_{g} Z^{g} - 1\bigg|&\leq 
    \frac{1}{2}\sum_{g=1}^{2}\bigg|\frac{1}{n_g\, m_g}{Z^{g}}^{\top} R_{g} Z^{g} - 1\bigg| \\
    &\leq \frac{1}{2}\sum_{g=1}^{2}\bigg(
    \frac{C}{n_g}\|R_g\|_{\text{op}} \sqrt{\frac{\log(\eta^{-1})}{m_g}}+\bigg| \frac{1}{n_g\,m_g} \text{tr}(R_{g})-1\bigg|\bigg)
\end{split}
 \end{equation}
 for some absolute constant $C>0$. We bound each term individually.

By Lemma~\ref{lem:NormBounds}, with probability at least $1-\eta/2:$
\begin{align*}
     C\sum_{g=1}^{2}  \frac{1}{n_g}\|R_g\|_{\text{op}} \sqrt{\frac{\log(\eta^{-1})}{m_g}}
     &\leq
      C\sum_{g=1}^{2}K_{s}^{2}\bigg[ 1+\sqrt{\frac{\log(\eta^{-1})}{n_g}}\bigg]\sqrt{\frac{\log(\eta^{-1})}{m_g}}
     \leq 
 C\, K_s^{2} \sqrt{\frac{\log(\eta^{-1})}{m}},
\end{align*}
where we have absorbed the lower-order term into the absolute constant $C$ and used Assumption~\ref{assump:EqualClass}.

Since $\text{tr}(R_{g})=\|Q^{g}/\sqrt{s}\|_{F}^{2}$, by Hoeffding's Inequality, Theorem 2.6.3 of \citet{vershynin2018high}, the following inequalities hold with probability at least $1-\eta/2$:
\begin{align*}
   \bigg| \frac{1}{n_g\,m_g} \text{tr}(R_{g})-1\bigg|&=\frac{1}{2}\sum_{g=1}^{2} \bigg|\frac{1}{n_g\, m_g}\bigg\|\frac{1}{\sqrt{s}}Q^{g}\bigg\|_{F}^{2}-1\bigg|
   = \frac{1}{2}\sum_{g=1}^{2}\bigg|  \frac{1}{n_g\, m_g}\sum_{i=1}^{n_g}\sum_{j=1}^{m_g} \bigg\{\bigg(\frac{1}{\sqrt{s}}Q_{i,j}^{g}\bigg)^{2}-1\bigg\}\bigg|&\\
   &\leq 
   \frac{1}{2}\sum_{g=1}^{2}K_{s}^{2}\sqrt{\frac{\log(\eta^{-1})}{n_g\,m_g}}= 
 2\,K_{s}^{2}\, \sqrt{\frac{\log(\eta^{-1}) }{n\,m} },
\end{align*}
 where Assumption~\ref{assump:EqualClass} was used in the last equality.

Combining the above two displays with~\eqref{eq:lemma2start}, and absorbing the lower order terms gives
$$
   \bigg|\frac{1}{2}\sum_{g=1}^{2}\frac{1}{n_g\, m_g} {Z^{g}}^{\top} R_{g} Z^{g} - 1\bigg|
   \leq    C\,
 K_s^{2} \,\sqrt{\frac{\log(\eta^{-1})}{m}}
$$
with probability at least $1-\eta$ for some absolute constant $C>0$.\qedhere
\end{proof}

\vspace{.2in}

\begin{lemma}[Norm Bound]\label{lem:NormBounds}
Let $Q\in \mathbb{R}^{m\times n}$ be a matrix consisting of i.i.d. sparse Rademacher random variables with sparsity parameter $s$, and let $R= {Q}^\top Q/s$. Then with probability at least $1-\eta$:
\begin{align*}
    &\frac{\|R\|_{\text{op}}}{n\,m}\leq C \frac{K_s^2}{m}\bigg[ 1+\sqrt{\frac{\log(\eta^{-1})}{n}}\bigg],
\end{align*}
where $C>0$ is an absolute constant, and $K_{s} = \{s \log(1+s^{-1})\}^{-1/2}$ is the sub-gaussian norm of $Q_{i,j}/\sqrt{s}$.
\end{lemma}

\begin{proof}[Proof of Lemma~\ref{lem:NormBounds}]
By Lemma~\ref{lem:sub-GaussianNorm}, $K_{s} = \{s \log(1+s^{-1})\}^{-1/2}$ is the sub-Gaussian norm of $Q_{i,j}/\sqrt{s}$. By Theorem 4.4.5 of \citet{vershynin2018high}, with probability at least $1-\eta:$

\begin{align*}
\|R\|_{\text{op}}&= \bigg\|\frac{1}{\sqrt{s}} Q \bigg\|_{\text{op}}^{2}\leq C K_s^2 (\sqrt{m}+\sqrt{n}+\sqrt{\log(\eta^{-1})})^{2}
\end{align*}
Including the scaling $(n\,m)^{-1}$ gives
\begin{align*}
 \frac{\|R\|_{\text{op}}}{n\,m}&=\frac{C K_s^2}{n\,m} (\sqrt{m}+\sqrt{n}+\sqrt{\log(\eta^{-1})})^{2}= 
 \frac{C\,K_{s}^{2}}{m} \bigg[\sqrt{\frac{m}{n}}+1 + \sqrt{\frac{\log(\eta^{-1})}{n}}\bigg]^{2}\\
 &\leq \frac{C K_{s}^{2}}{m}\bigg[ 2+ \sqrt{\frac{\log(\eta^{-1})}{n}}\bigg]^{2}
 \leq \frac{C K_{s}^{2}}{m}\bigg[ 1 +\sqrt{\frac{\log(\eta^{-1})}{n}}\bigg],
\end{align*}
where we have expanded the square and absorbed the lower-order terms into the absolute constant $C>0$. \qedhere
\end{proof}

\vspace{.2in}

\begin{lemma}[Conditional Hanson-Wright]\label{lem:ConditionalHW}
Let $Z\sim N(0, I_{n})$, and let $R\in\mathbb{R}^{n\times n}$ be a matrix of rank $m$. Conditioning on $R$, and for $\eta \in (0, e^{-1})$, the following upper bound holds with probability at least $1-\eta:$
\begin{align}
\begin{split}
&\bigg|\frac{1}{n\, m} Z^{\top} R Z - 1\bigg|\leq  \frac{C}{n}\|R\|_{\text{op}} \sqrt{\frac{\log(\eta^{-1})}{m}}+\bigg| \frac{1}{n\,m} \text{tr}(R)-1\bigg|,
\end{split}
\end{align}
where $C>0$ is an absolute constant. 
\end{lemma}

\begin{proof}[Proof of Lemma~\ref{lem:ConditionalHW}]
Since $Z\sim N(0, I_{n})$, the conditional expectation equals
$$
\mathbb{E}[Z^{\top}R Z\,|\,R] = \text{tr}(R\,I_{n})+0^\top R \,0= \text{tr}(R).
$$
The Hanson-Wright Inequality, Theorem 6.2.1 of \citet{vershynin2018high}, gives the conditional tail bound
\begin{align*}
\mathbb{P}(|Z^{\top} R Z -\text{tr}(R)\,|\,\geq t\, n\,m\,|\, R)
&=\mathbb{P}(|Z^{\top} R Z- \mathbb{E}[Z^\top R \,Z\,|\, R]|\geq t\,n\,m\,|\, R)\\
&\leq2 \exp\bigg(-C\,\text{min}\bigg( \frac{t^2\, m^{2}\,n^2}{\|R\|^{2}_{F}}\,,\, \frac{t\, m\,n}{\|R\|_{\text{op}}}\bigg) \bigg)
\end{align*}
for some absolute $C>0$. This is equivalent to the following upper bound holding with probability at least $1-\eta$:
$$
\frac{1}{n\, m}|Z^{\top} R Z-\text{tr}(R)|\leq 
 \frac{C}{m \, n} \max\bigg\{\|R\|_{F} \sqrt{\log(\eta^{-1})}\,,\, \|R\|_{\text{op}}\log(\eta^{-1})\bigg\}.
$$
 Using the fact that $\|R\|_{F}\leq \sqrt{m}\|R\|_{\text{op}}$ and $m\geq \log(\eta^{-1})$ for $\eta \leq e^{-1}$, this is further bounded by
\begin{equation}\label{eq:HWconditional}
\begin{split}
 \frac{1}{n\, m}|Z^{\top} R Z-\text{tr}(R)|&\leq \frac{C}{m\, n}
 \max\bigg\{ \sqrt{m}\|R\|_{\text{op}} \sqrt{\log(\eta^{-1})}\,,\, \|R\|_{\text{op}}\log(\eta^{-1})\bigg\}\\
 &\leq
 \frac{C}{n} \|R\|_{\text{op}}\max\bigg\{ \sqrt{\frac{\log(\eta^{-1})}{m}}\,,\,
 \frac{\log(\eta^{-1})}{m}\bigg\} = \frac{C}{n}\|R\|_{\text{op}} \sqrt{\frac{\log(\eta^{-1})}{m}}.
 \end{split}
 \end{equation}

Applying the triangle inequality and substituting \eqref{eq:HWconditional} gives the final result:
 \begin{align*}
   \bigg|\frac{1}{n\, m}Z^{\top}RZ-1\bigg|&
   \leq \bigg|\frac{1}{n\,m}Z^{\top}RZ-\frac{1}{n\,m}\,\text{tr}(R)\bigg|+ \bigg|\frac{\text{tr}(R)}{n\,m}- 1\bigg|\\
    &
     \leq  \frac{C}{n}\|R\|_{\text{op}} \sqrt{\frac{\log(\eta^{-1})}{m}}+\bigg| \frac{\text{tr}(R)}{n\,m}-1\bigg|.\qedhere
\end{align*}
\end{proof}

\begin{lemma}[Sub-Gaussian Norm]\label{lem:sub-GaussianNorm}
Let $X$ be sparse Rademacher random variable satisfying for some $s\in (0,1)$
$$
P(X = 0)=1-s,\quad P(X = 1) =P(X = -1)= s/2.
$$
Then the sub-Gaussian norm of $X$ is $K = \{\log(1 + s^{-1})\}^{-1/2}$, and the sub-Gaussian norm of $X/\sqrt{s}$ is $K_s = \{s\log(1 + s^{-1})\}^{-1/2}$. Additionally, the sub-Gaussian norm of $X^2/s$ is $s^{-1} \{\log(1+s^{-1})\}^{-1/2}.$
\end{lemma}
\begin{proof}
By definition of sub-Gaussian norm,
$$
K = \inf\{t>0: \mathbb{E}\exp(X^2/t^2)\leq 2\}.
$$
Consider for some $t>0$,
\begin{align*}
    \mathbb{E}\exp(X^2/t^2) = \exp(0/t^2)(1-s) + \exp(1/t^2)s= 1-s + \exp(1/t^2)s.
\end{align*}
Then $\mathbb{E}\exp(X^2/t^2) \leq 2$ is equivalent to
\begin{align*}
    1-s + \exp(1/t^2)s&\leq 2\\
    \exp(1/t^2)s&\leq 1 + s\\
    \exp(1/t^2) &\leq 1 + s^{-1}\\
    1/t^2 &\leq \log(1+ s^{-1})\\
    t^2&\geq \{\log(1+ s^{-1})\}^{-1}.
\end{align*}
The term $K_s = \{s \log(1+s^{-1})\}^{-1/2}$ follows from scaling $X$ by $s^{-1/2}.$

Additionally, the sub-gaussian norm of the squared $X^2/s$ is
$
\big\| X^2 / s \|_{\Psi_2} = \|X^2\|_{\Psi_2}/s.
$
Because $X$ has values $0$ and $\pm 1$, it follows that $X^{4}=X^{2}$. Thus,
\begin{align*}
     \mathbb{E}\exp(X^{4}/t^2)&\leq 2\\
     \mathbb{E}\exp(X^{2}/t^2)& \leq 2\\
     \exp(1/t^2)s &\leq 1+s\\
     \exp(1/t^2)&\leq 1+ s^{-1}\\
     t/t^{2}&\leq \log(1+s^{-1})\\
     t^{2}&\geq \{\log(1+s^{-1})\}^{-1}.
\end{align*}
Hence, the sub-gaussian norm of $X^2/s$ is $s^{-1}\{\log(1+s^{-1})\}^{-1/2}.$ \qedhere

\end{proof}

\section{Extension of Theorem 1 to unequal class sizes}\label{appen:Extension}
Assumption~\ref{assump:EqualClass} simplifies the statement of Theorem~\ref{thm:MisclassificationError} and reduces the technical complexity of the proofs. An analogous result with the same rate of convergence holds in the setting of unequal class sizes $n_1 \neq n_2$ and unequal class prior probabilities $\pi_1 \neq \pi_2$, but with more complicated expressions for the constants. This section outlines the adjustments necessary for extending Theorem~\ref{thm:MisclassificationError} to the general case. 

The Bayes error rate $R_{\text{opt}}$ under Assumption~\ref{assump:Normality}, but with $\pi_1\neq \pi_2$, is
\begin{align*}
R_{\text{opt}}&:= \bigg\{ 1- \Phi\bigg(\frac{(\mu_2-\mu_1)^\top\Sigma_{w}^{-1}(\mu_2 - \mu_1)/2 + \log(\pi_1/\pi_2) }{\sqrt{(\mu_2-\mu_1)^\top\Sigma_{w}^{-1}(\mu_2-\mu_1)}}\bigg)\bigg\}\pi_1\\
&\qquad\qquad+\Phi\bigg(\frac{-(\mu_2-\mu_1)^\top\Sigma_{w}^{-1}(\mu_2-\mu_1)/2 + \log(\pi_1/\pi_2)}{
\sqrt{(\mu_2-\mu_1)^\top\Sigma_{w}^{-1}(\mu_2-\mu_1)}
}\bigg)\pi_2.
\end{align*}
The misclassification error rate of decision rule \eqref{eq:ClassificationRule} using the compressed discriminant vector $\beta_{\mathbf{c}}$ is
\begin{align*}
R_{\mathbf{c}}&= \bigg\{1-\Phi\bigg(\frac{\beta_{\mathbf{c}}^\top (\overline{X}_{1}+\overline{X}_{2})/2 + \log(n_1/n_2) - \beta_{\mathbf{c}}^\top\mu_1}{\sqrt{\beta_\mathbf{c}^\top \Sigma_w \beta_{\mathbf{c}} }}\bigg)\bigg\} \pi_1\\
&\qquad\qquad+ 
\Phi\bigg(\frac{\beta^\top (\overline{X}_{1}+\overline{X}_{2})/2 + \log(n_1/n_2) - \beta_{\mathbf{c}}^\top\mu_2}{\sqrt{\beta_{\mathbf{c}}^\top \Sigma_w \beta_{\mathbf{c}} }}\bigg) \pi_2.
\end{align*}

We now outline the main changes in the proof of Theorem~\ref{thm:MisclassificationError} which guarantees that $|R_{\mathbf{c}} - R_{\text{opt}}|$ converges to $0$ with the same rate under the setting $\pi_1 \neq \pi_2$. As with Theorem~\ref{thm:MisclassificationError}, taking the Taylor expansion leads to $R_{\mathbf{c}} = R_{\text{opt}}+\text{ error terms}$, where the error terms depend on the difference of the sample estimates $\widehat{\Sigma}_{w, \mathbf{c}}$, $\overline{X}_{g}$, and $n_g/n$ from the corresponding population values $\Sigma_{w}$, $\mu_g$, and $\pi_g$. Since the sample terms are the same, their rates of convergence to the populations parameters are established in Theorems~\ref{thm:MeanInvCovariance} through \ref{thm:SigmaBound}. The new term $\log(n_1/n_2)$ converges to $\log(\pi_1/\pi_2)$ with $n^{-1/2}$ rate as a consequence of Hoeffding's inequality, see e.g. Lemma~11 in \citet{gaynanova2020prediction}.



\section{Additional simulation studies}

\subsection{Compression Matrix Comparison}
In this section, we investigate the use of different sparse compression matrices. We consider sparse Rademacher, sparse Gaussian, and count sketch matrices for compressed LDA with the same sparsity level $s=0.01$ and repeat the MNIST simulation in Section~\ref{sec:Simulations}.
Figure \ref{fig:MatrixComparison} displays the error rates across $100$ independent replications compressed LDA using three different compression matrices and different reduced sample amounts $m$ on the MNIST Data. The error rates across all compression levels $m$ are nearly indistinguishable. For instance, at $m= 2,000$, sparse Rademacher matrices have a mean error rate of $ 14.01\%$ (se $0.04\%$), sparse Gaussian matrices have a mean error rate of $14.00\%$ (se $0.04\%$), and count sketch matrices have a mean error rate of $13.94\%$ (se $0.04\%$). This shows that the sparsity parameter determines predictive accuracy more than the distribution of non-zero elements.

\begin{figure}
\centering
\includegraphics[width=4in]{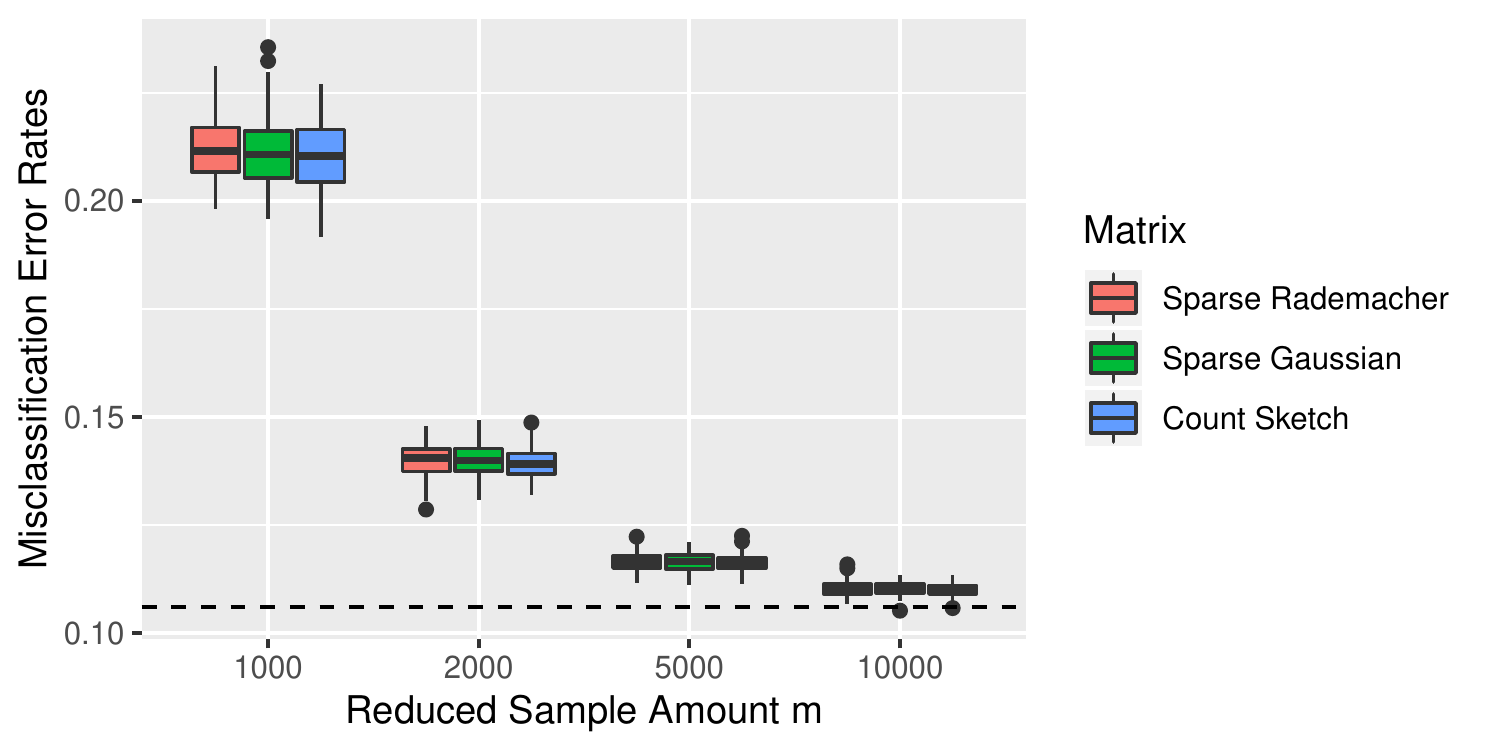} 
\caption{MNIST Data.  Misclassification error rates across 100 replications of compressed LDA using different compression matrices for each value of $m$ with $s=0.01$ and $\gamma=10^{-3}$.  The dashed line represents the $10.60\%$ error rate of Full LDA. }
\label{fig:MatrixComparison}
\end{figure} 

\subsection{Skewed MNIST Data}
In this section we investigate the performance of the proposed methods when the equal class size Assumption~\ref{assump:EqualClass} is violated. That is, we sub-sample MNIST data to have $1/3$-rd of class 1 and the remaining $2/3$-rds of class 2. Figure~\ref{fig:SkewedMNIST} displays the error rates across $100$ independent iterations of the linear classifiers. Compressed LDA consistently has the lowest error rates across all tested compression levels $m$. For example, at $m=2,000$, compressed LDA has an error rate of $14.36\%$ (se $0.03\%$), while projected LDA has $ 15.40\%$ (se $0.06\%$), and sub-sampled LDA has $16.56\%$ (se $0.05\%$). FRF has the worst classification performance across all values of $m$. For instance, at $m=2,000$ its mean error rate is $17.61\%$ (se $0.05\%$). Thus, when the class proportions are skewed, the proposed methods have lower misclassification error rates compared to FRF and sub-sampling.   

\begin{figure}
\centering
\includegraphics[width=4in]{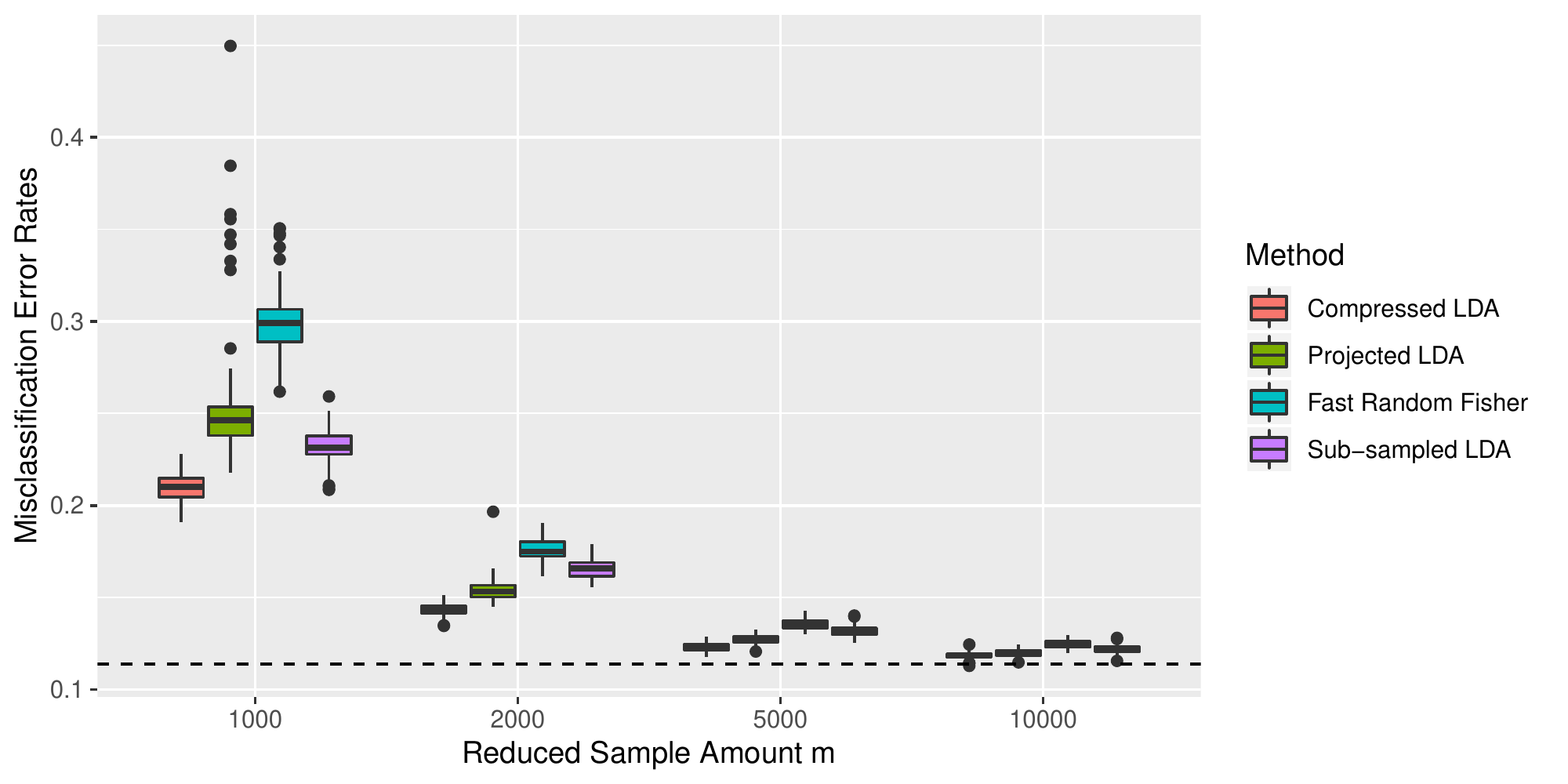} 
\caption{Skewed MNIST Data.  Misclassification error rates across 100 replications of linear methods $(i)-(v)$ across different values of $m$ with $s=0.01$ and $\gamma=10^{-3}$. The training and testing data was skewed so that class $1$ comprises $33.3\%$ of the data and class $2$ the remaining $66.6\%$. The dashed line represents the $11.4\%$ error rate of Full LDA. }
\label{fig:SkewedMNIST}
\end{figure}

\subsection{Sparsity Comparison Simulation} 
In this Section, we investigate the influence of the sparsity parameter $s$ on the performance of the linear classifiers. We rerun the MNIST simulation of Section~\ref{sec:Simulations} with $m=5,000$ fixed reduced samples, but varying sparsity level $s$.
Figure~\ref{fig:MNISTsparsity} displays the error rates for compressed LDA, projected LDA, and FRF across 100 independent iterations. The error rates appear to be stable for sparsity levels between $10^{-1}$ and $10^{-3}$. For such $s$, both compressed and projected LDA have uniformly lower error rates compared to FRF. For example, at $s=10^{-3}$, compressed LDA has mean error rate $11.67\%$ (se $0.02\%$) and projected LDA has mean error rate $11.73\%$ (se $0.03\%$). This is compared to FRF's mean error rate $12.33\%$ (se $0.03\%$). 

However, for the very sparse $s=10^{-4}$, projected LDA has significantly larger error rates compared to either compressed LDA or FRF. We hypothesize this is because the sparsity level is so small that for a significant number of compressed samples, no training samples are used in construction leading to effectively smaller $m$. FRF partially avoids this because it draws from both classes simultaneously when forming each compressed sample. Thus, it draws from a larger pool of samples and is more likely to sample at least one training sample. However, FRF doesn't avoid this problem entirely because there are large outliers in the error rates for $s=10^{-4}$.

\begin{figure}
\centering
\includegraphics[width=4in]{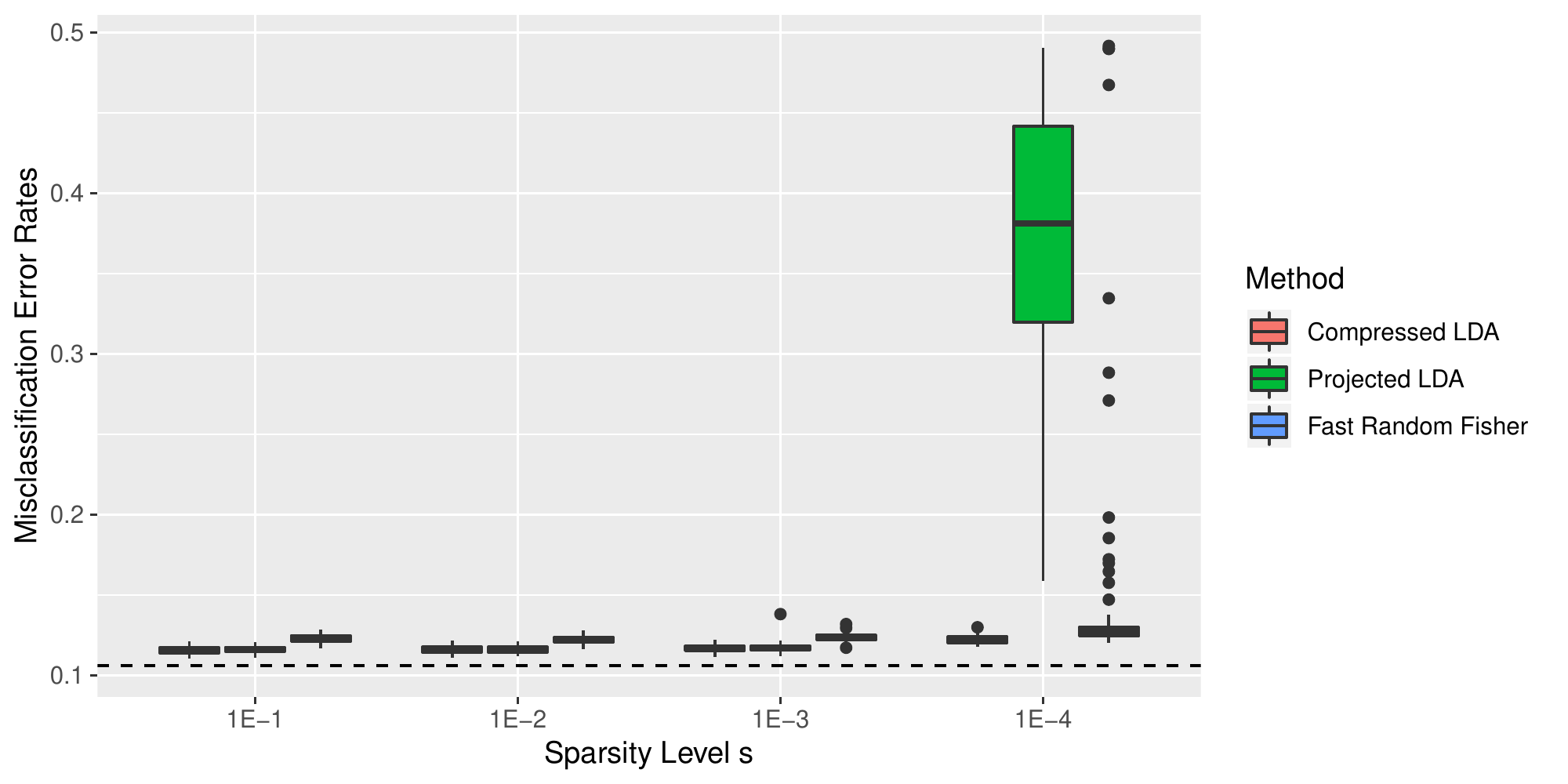} 
\caption{MNIST Data.  Misclassification error rates across 100 replications of compressed LDA, projected LDA, and FRF across different sparsity levels $s$ with $m=5,000$ reduced samples and $\gamma=10^{-3}$. The dashed line represents the $10.60\%$ error rate of Full LDA. }
\label{fig:MNISTsparsity}
\end{figure}

\subsection{Heavy-tailed Data}
To investigate the behavior of the linear classification methods when the normality Assumption~\ref{assump:Normality} is violated, we generate $n=10,000$ samples from a multivariate $t$-distribution with $p=100$ features and $5$ degrees of freedom. The class covariance matrices have coordinates $(\Sigma_{w})_{i,j}=(0.9)^{|i-j|}$, and the class means are $\mu_g = (-1)^{g}\mathbf{1}$. 

\begin{figure}[!t]
\centering
\includegraphics[width=4in]{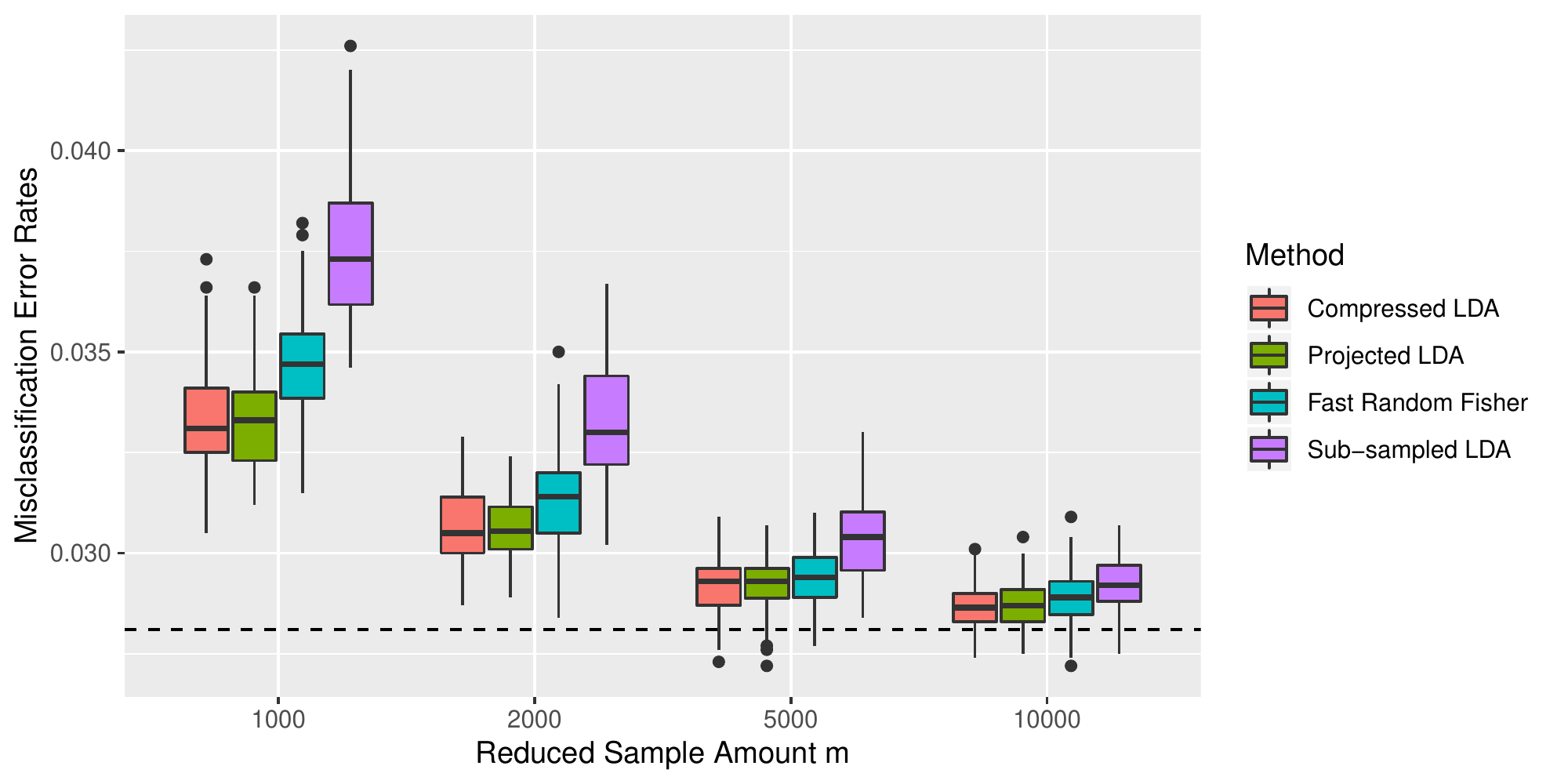} 
\caption{Heavy-tailed Data.  Misclassification error rates across 100 replications of compressed LDA, projected LDA, FRF, and sub-sampled LDA across different smaple amounts $m=5,000$ for $s=0.01$ and $\gamma=10^{-3}$. The dashed line represents the $2.81\%$ error rate of Full LDA. }
\label{fig:HeavyTail}
\end{figure}

Figure~\ref{fig:HeavyTail} displays the error rates across $100$ independent iterations of the linear classifiers. Compressed and projected LDA consistently have the lowest error rates while sub-sampling has the highest error rates. For example, at $m=2,000$ samples, compressed LDA has a mean error rate of $3.06\%$ (se $0.01\%$), and projected LDA has a mean error rate of $3.06\%$ (se $0.01\%$). This is compared with FRF which has a mean error rate of $3.1\%$ (se $0.01\%$), and sub-sampled LDA which has a mean error rate of $3.32\%$ (se $0.01\%$). This suggests that sub-sampling outliers (a result of the heavy tails) influences FRF and sub-sampled LDA's performance more compared to compressed and projected LDA. We suspect the reason is that our proposed compression scheme appears to induce normality, and so the proposed compressed and projected LDA are robust to heavy tails.

\end{appendix}


\bibliography{main}

\begin{thebibliography}{}

\bibitem[Bhatt and Dhall, 2010]{bhatt2010skin}
Bhatt, R. and Dhall, A. (2010).
\newblock Skin segmentation dataset.
\newblock {\em UCI Machine Learning Repository}.

\bibitem[Bickel et~al., 2004]{bickel2004some}
Bickel, P.~J., Levina, E., et~al. (2004).
\newblock Some theory for fisher's linear discriminant function,naive bayes',
  and some alternatives when there are many more variables than observations.
\newblock {\em Bernoulli}, 10(6):989--1010.

\bibitem[Boutsidis and Drineas, 2009]{boutsidis2009random}
Boutsidis, C. and Drineas, P. (2009).
\newblock Random projections for the nonnegative least-squares problem.
\newblock {\em Linear algebra and its applications}, 431(5-7):760--771.

\bibitem[Chowdhury et~al., 2018]{chowdhury2018randomized}
Chowdhury, A., Yang, J., and Drineas, P. (2018).
\newblock Randomized iterative algorithms for fisher discriminant analysis.
\newblock {\em arXiv preprint arXiv:1809.03045}.

\bibitem[Drineas et~al., 2011]{drineas2011faster}
Drineas, P., Mahoney, M.~W., Muthukrishnan, S., and Sarl{\'o}s, T. (2011).
\newblock Faster least squares approximation.
\newblock {\em Numerische mathematik}, 117(2):219--249.

\bibitem[DuMouchel et~al., 1999]{dumouchel1999squashing}
DuMouchel, W., Volinsky, C., Johnson, T., Cortes, C., and Pregibon, D. (1999).
\newblock Squashing flat files flatter.
\newblock In {\em Proceedings Of ACM SIGKDD}, volume~15, pages 6--15.

\bibitem[Durrant and Kab{\'a}n, 2010]{durrant2010compressed}
Durrant, R.~J. and Kab{\'a}n, A. (2010).
\newblock Compressed fisher linear discriminant analysis: Classification of
  randomly projected data.
\newblock In {\em Proceedings of the 16th ACM SIGKDD international conference
  on Knowledge discovery and data mining}, pages 1119--1128.

\bibitem[Durrant and Kab{\'a}n, 2012]{durrant2012tight}
Durrant, R.~J. and Kab{\'a}n, A. (2012).
\newblock A tight bound on the performance of fisher’s linear discriminant in
  randomly projected data spaces.
\newblock {\em Pattern Recognition Letters}, 33(7):911--919.

\bibitem[Friedman et~al., 2009]{hastie_elements_2009}
Friedman, J., Hastie, T., and Tibshirani, R. (2009).
\newblock {\em The Elements of statistical learning}.
\newblock Springer Series in Statistics, New York, 2nd edition.

\bibitem[Gaynanova, 2020]{gaynanova2020prediction}
Gaynanova, I. (2020).
\newblock Prediction and estimation consistency of sparse multi-class penalized
  optimal scoring.
\newblock {\em Bernoulli}, 26(1):286--322.

\bibitem[Homrighausen and McDonald, 2019]{homrighausen2019compressed}
Homrighausen, D. and McDonald, D.~J. (2019).
\newblock Compressed and penalized linear regression.
\newblock {\em Journal of Computational and Graphical Statistics}, 00(0):1--14.

\bibitem[Hsu et~al., 2012]{hsu2012tail}
Hsu, D., Kakade, S., Zhang, T., et~al. (2012).
\newblock A tail inequality for quadratic forms of subgaussian random vectors.
\newblock {\em Electronic Communications in Probability}, 17.

\bibitem[LeCun et~al., 1998]{lecun1998gradient}
LeCun, Y., Bottou, L., Bengio, Y., Haffner, P., et~al. (1998).
\newblock Gradient-based learning applied to document recognition.
\newblock {\em Proceedings of the IEEE}, 86(11):2278--2324.

\bibitem[Li et~al., 2019]{li2019one}
Li, W.-H., Zhong, Z., and Zheng, W.-S. (2019).
\newblock One-pass person re-identification by sketch online discriminant
  analysis.
\newblock {\em Pattern Recognition}, 93:237--250.

\bibitem[Madigan et~al., 2002]{madigan2002likelihood}
Madigan, D., Raghavan, N., Dumouchel, W., Nason, M., Posse, C., and Ridgeway,
  G. (2002).
\newblock Likelihood-based data squashing: A modeling approach to instance
  construction.
\newblock {\em Data Mining and Knowledge Discovery}, 6(2):173--190.

\bibitem[Mahoney et~al., 2011]{mahoney2011randomized}
Mahoney, M.~W. et~al. (2011).
\newblock Randomized algorithms for matrices and data.
\newblock {\em Foundations and Trends{\textregistered} in Machine Learning},
  3(2):123--224.

\bibitem[Mardia et~al., 1979]{mardia79}
Mardia, K.~V., Kent, J.~T., and Bibby, J.~M. (1979).
\newblock {\em Multivariate analysis}.
\newblock Academic Press, Orlando, FL.

\bibitem[McLachlan, 2004]{mclachlan2004discriminant}
McLachlan, G. (2004).
\newblock {\em Discriminant analysis and statistical pattern recognition},
  volume 544.
\newblock John Wiley \& Sons, Hoboken, NJ.

\bibitem[Mika et~al., 1999]{mika1999fisher}
Mika, S., Ratsch, G., Weston, J., Scholkopf, B., and Mullers, K.-R. (1999).
\newblock Fisher discriminant analysis with kernels.
\newblock In {\em Neural networks for signal processing IX: Proceedings of the
  1999 IEEE signal processing society workshop}, pages 41--48. IEEE.

\bibitem[Pavlov et~al., 2000]{pavlov2000towards}
Pavlov, D., Chudova, D., and Smyth, P. (2000).
\newblock Towards scalable support vector machines using squashing.
\newblock In {\em Proceedings of the sixth ACM SIGKDD international conference
  on Knowledge discovery and data mining}, volume~20, pages 295--299, Boston,
  MA.

\bibitem[Pilanci and Wainwright, 2015]{pilanci2015randomized}
Pilanci, M. and Wainwright, M.~J. (2015).
\newblock Randomized sketches of convex programs with sharp guarantees.
\newblock {\em IEEE Transactions on Information Theory}, 61(9):5096--5115.

\bibitem[Pilanci and Wainwright, 2016]{pilanci2016iterative}
Pilanci, M. and Wainwright, M.~J. (2016).
\newblock Iterative hessian sketch: Fast and accurate solution approximation
  for constrained least-squares.
\newblock {\em The Journal of Machine Learning Research}, 17(1):1842--1879.

\bibitem[R{\"o}sler and Suendermann, 2013]{rosler2013first}
R{\"o}sler, O. and Suendermann, D. (2013).
\newblock A first step towards eye state prediction using eeg.
\newblock {\em Proc. of the AIHLS}.

\bibitem[Sarlos, 2006]{sarlos2006improved}
Sarlos, T. (2006).
\newblock Improved approximation algorithms for large matrices via random
  projections.
\newblock In {\em Proceedings of the 47th Annual {IEEE} Symposium on
  Foundations of Computer Science (FOCS)}, pages 143--152. IEEE.

\bibitem[Shao et~al., 2011]{shao2011sparse}
Shao, J., Wang, Y., Deng, X., Wang, S., et~al. (2011).
\newblock Sparse linear discriminant analysis by thresholding for high
  dimensional data.
\newblock {\em The Annals of Statistics}, 39(2):1241--1265.

\bibitem[Tu et~al., 2014]{tu2014making}
Tu, B., Zhang, Z., Wang, S., and Qian, H. (2014).
\newblock Making fisher discriminant analysis scalable.
\newblock In {\em International Conference on Machine Learning}, pages
  964--972.

\bibitem[Vempala, 2005]{vempala2005random}
Vempala, S.~S. (2005).
\newblock {\em The random projection method}, volume~65.
\newblock American Mathematical Society, Providence, RI.

\bibitem[Vershynin, 2018]{vershynin2018high}
Vershynin, R. (2018).
\newblock {\em High-dimensional probability: An introduction with applications
  in data science}, volume~47.
\newblock Cambridge University Press.

\bibitem[Wang et~al., 2017]{wang2017sketched}
Wang, S., Gittens, A., and Mahoney, M.~W. (2017).
\newblock Sketched ridge regression: Optimization perspective, statistical
  perspective, and model averaging.
\newblock {\em The Journal of Machine Learning Research}, 18(1):8039--8088.

\bibitem[Ye et~al., 2017]{ye2017fast}
Ye, H., Li, Y., Chen, C., and Zhang, Z. (2017).
\newblock Fast fisher discriminant analysis with randomized algorithms.
\newblock {\em Pattern Recognition}, 72:82--92.

\bibitem[Zhou et~al., 2008]{zhou2008compressed}
Zhou, S., Wasserman, L., and Lafferty, J.~D. (2008).
\newblock Compressed regression.
\newblock In {\em Advances in Neural Information Processing Systems}, pages
  1713--1720.

\end{thebibliography}

\end{document}